\documentclass{article}

 \usepackage[final]{neurips_2019}

\usepackage[utf8]{inputenc} 
\usepackage[T1]{fontenc}    
\usepackage{hyperref}       
\usepackage{url}            
\usepackage{booktabs}       
\usepackage{amsfonts}       
\usepackage{nicefrac}       
\usepackage{microtype}      
\usepackage{color}
\usepackage{natbib}

\usepackage{qiangstyle}
\usepackage{setspace}

\usepackage{graphicx}
\usepackage{tabu}

\usepackage{subcaption}

\title{
Splitting  Steepest Descent 
for Growing Neural Architectures 
}
\author{
{\rm Qiang Liu}\\
{\rm UT Austin}\\
{\rm \texttt{lqiang@cs.utexas.edu}} 
\and
{\rm Lemeng Wu \thanks{equal contribution}}\\
{\rm UT Austin}\\
{\rm \texttt{lmwu@cs.utexas.edu}} 
\and
{\rm  Dilin Wang \textsuperscript{*}}\\
{\rm UT Austin}\\
{\rm \texttt{dilin@cs.utexas.edu}} 
}
\renewcommand{\todo}[1]{}

\renewcommand{\red}[1]{}

\begin{document}

\maketitle

\begin{abstract}
%
%
We develop a progressive training approach for neural networks which adaptively grows the network structure by splitting existing neurons to multiple off-springs. 
By leveraging a functional steepest descent idea,  
we derive a simple criterion for deciding the best subset of neurons to split 
and a \emph{splitting gradient} for optimally updating the off-springs. 
Theoretically, our splitting strategy is a second-order functional steepest descent for escaping saddle points in an $\Linfty$-Wasserstein metric space, 
on which the standard parametric gradient descent is a first-order steepest descent.  
Our method provides a new practical approach for 
optimizing 
neural network structures,  
especially for learning lightweight neural architectures in resource-constrained settings. 
\end{abstract}
\vspace{-1em}

\section{Introduction}


Deep neural networks (DNNs) have achieved remarkable empirical successes recently. 
However, efficient and automatic optimization of model architectures remains to be a key challenge. 
Compared with parameter optimization which has been well addressed by gradient-based methods (a.k.a. back-propagation), 
optimizing model structures involves significantly more  
challenging discrete optimization with large search spaces and high evaluation cost. 
%
Although there have been  rapid progresses recently,  
designing  the best architectures still requires a lot of expert knowledge and trial-and-errors for most practical tasks.

This work 
targets 
extending the power of gradient descent to the domain of model structure optimization of neural networks.  
In particular, we consider the problem of progressively  
growing a neural network by ``splitting'' existing neurons into several ``off-springs'',  
and develop a simple and practical approach for deciding 
the best subset of neurons to split and how to split them, adaptively based on the existing model structure.  
We derive the optimal splitting strategies by
considering the \emph{steepest descent} of the loss when  
 the off-springs are 
 infinitesimally close to the original neurons,
 yielding a \emph{splitting steepest descent} 
 that monotonically decreases the loss 
 in the space of model structures. 
%

Our main method, shown in Algorithm~\ref{alg:main}, 
alternates between 
a standard \emph{parametric descent phase} in which we update the parameters to minimize the loss with a fixed model structure, 
and a \emph{splitting phase} 
in which we update the model structures by splitting neurons. 
The splitting phase is triggered when 
no further improvement can be made by only updating parameters, and allow us to escape the parametric local optima by augmenting the neural network in a locally optimal fashion. 
%
Theoretically, these two phases can be viewed as performing functional steepest descent on an $\Linfty$-Wasserstein metric space, in which the splitting phase is a \emph{second-order descent} for escaping saddle points in the functional space, while the parametric gradient descent corresponds to a \emph{first-order descent}. 
Empirically, our algorithm is simple and practical, 
and provides a promising tool for many challenging problems, especially for 
learning lightweight and energy-efficient neural architectures for resource-constrained settings.

\paragraph{Related Works}  
 The idea of progressively growing neural networks by node splitting is not new, 
 but 
 previous works are mostly based on heuristic or purely random splitting strategies \citep[e.g.,][]{wynne1992node, chen2015net2net}. 
 A different  approach for progressive training is  
 the  Frank-Wolfe or gradient boosting  based strategies \citep[e.g.,][]{schwenk2000boosting,bengio2006convex, bach2017breaking}, 
 which iteratively add new neurons derived from functional conditional gradient, while keeping the previous neurons fixed. 
However, these methods are not suitable for large scale settings,    
because adding each neuron requires to solve a difficult non-convex optimization problem, 
and keeping the previous neurons fixed prevents us from correcting the mistakes made in earlier iterations.   
A practical alternative of  Frank-Wolfe is to simply add new randomly initialized neurons and co-optimize the new and old neurons together. 
However, random initialization 
does not allow us to leverage the information of the existing model and takes more time to converge.  
 In contrast, 
 splitting neurons from the existing network allows us to inherent the knowledge from the existing model (see \citet{chen2015net2net}), and 
 is faster to converge in settings like continual learning, when the previous model is not far away from the optimal solution.  

An opposite direction of progressive training is to \emph{prune} large pre-trained neural networks to obtain compact network structures \citep[e.g.,][]{han2015deep, li2016pruning, liu2017learning}. 
In comparison, our splitting method requires no large pre-trained models
and is more suitable 
for learning \emph{very small}  network structures, which is of critical importance  for resource-constrained settings like mobile devices and Internet of things.
As shown in our experiments, 
our method can outperform existing pruning methods in  
learning more accurate models with small model sizes. 

More broadly, there has been a series of recent works on 
neural architecture search, based on various strategies from combinatorial optimization, 
including reinforcement learning (RL)~\citep[e.g.,][]{pham2018efficient, cai2018proxylessnas, zoph2016neural},
evolutionary algorithms (EA)~\citep[e.g.,][]{stanley2002evolving, real2018regularized},
and continuous relaxation~\citep[e.g.,][]{liu2018darts, xie2018snas}. 
However, these general-purpose black-box optimization methods do not efficient leverage the inherent geometric structure of the loss landscape, and are highly computationally expensive due to the need of evaluating the candidate architectures based on  inner training loops. 

\paragraph{Background: Steepest Descent and Saddle Points} 
Stochastic gradient descent is the driving horse for solving large scale optimization in machine learning and deep learning. 
Gradient descent can be viewed as a steepest descent procedure that iteratively improves the  solution by following the direction that maximally decreases the loss function within a small neighborhood of the previous solution.   
Specifically, for minimizing a loss function $L(\theta)$, 
each iteration of steepest descent updates the parameter via $\theta \gets \theta + \epsilon \delta$, where $\epsilon$ is a small step size and $\delta$ is an update direction chosen to maximally decrease the loss $L(\theta + \epsilon \delta)$ of the updated parameter under a norm constraint $\norm{\delta}\btmptwo\leq 1$,  
where $\norm{\cdot}$ denotes the Euclidean norm.   
When $\nabla L(\theta)\neq 0$ and $\epsilon$ is infinitesimal, 
the optimal descent direction $\delta$ equals the negative gradient direction, that is,  
$\delta = -\nabla L(\theta)/\norm{\nabla L(\theta)}$, 
 yielding a descent of $
L(\theta + \epsilon \delta)
- L(\theta) \approx -\epsilon \norm{\nabla L(\theta)}$. 
%
At a critical point with a zero gradient ($\nabla L(\theta)=0$), 
the steepest descent direction depends on the spectrum of the Hessian matrix $\nabla^2 L(\theta)$. Denote by $\lambda_{min}$ the minimum eigenvalue of $\nabla^2 L(\theta)$ and $v_{min}$ its associated eigenvector. 
When $\lambda_{min}>0$, the point $\theta$ is a stable local minimum and no further improvement can be made in the infinitesimal neighborhood.  
When $\lambda_{min} < 0$, 
the point $\theta$ is a saddle point or local maximum,  
and the steepest descent direction equals the eigenvector $\pm v_{min}$, 
which yields an $\epsilon^2\lambda_{min}/2$ decrease on the loss.\footnote{The property of the case when  $\lambda_{min}=0$ depends on higher order information.} 
In practice, it has been shown that there is no need to explicitly calculate the negative eigenvalue direction, because saddle points and local maxima are unstable and can be escaped by using gradient descent with random initialization or stochastic noise  \citep[e.g.,][]{lee2016gradient, jin2017escape}.  

\section{Splitting Neurons Using Steepest Descent} \label{sec:main}
We introduce our main method in this section. 
We first illustrate the idea with the simple case of splitting a single neuron 
in Section~\ref{sec:single}, 
and then consider the more general case of simultaneously splitting multiple neurons in deep networks in Section~\ref{sec:general}, which yields our main progressive training algorithm (Algorithm~\ref{alg:main}).  
Section~\ref{sec:wasser} draws a theoretical discussion and interpret our procedure as a functional steepest descent of the distribution of the neuron weights under the $\Linfty$-Wasserstein metric. 

\subsection{Splitting a Single Neuron}\label{sec:single}
Let $\sigma(\theta, x)$ be a neuron 
inside a neural network that we want to learn from data, where $\theta$ is the parameter of the neuron and $x$  its input variable. 
Assume the loss of $\theta$ has a general form of 
\begin{align}\label{equ:Ltheta}
L(\theta) :=  \E_{x\sim\mathcal D}[\Phi(\sigma(\theta, x))], 
\end{align}
where $\mathcal D$ is a data distribution, and $\Phi$ is a map determined by the overall loss function. 
The parameters of the other parts of the network are assumed to be fixed or optimized using standard procedures and  are omitted for notation convenience.  

Standard gradient descent can only 
yield parametric updates of $\theta$. 
We introduce a generalized steepest descent procedure that allows us to incrementally grow the neural network by gradually introducing new neurons, 
achieved by ``splitting'' the existing neurons into multiple copies 
in a (locally) optimal fashion derived using ideas from steepest descent idea. 


\begin{wrapfigure}{r}{0.4\textwidth}
  \vspace{-20pt}
  \begin{center}
    \includegraphics[width=0.38\textwidth]{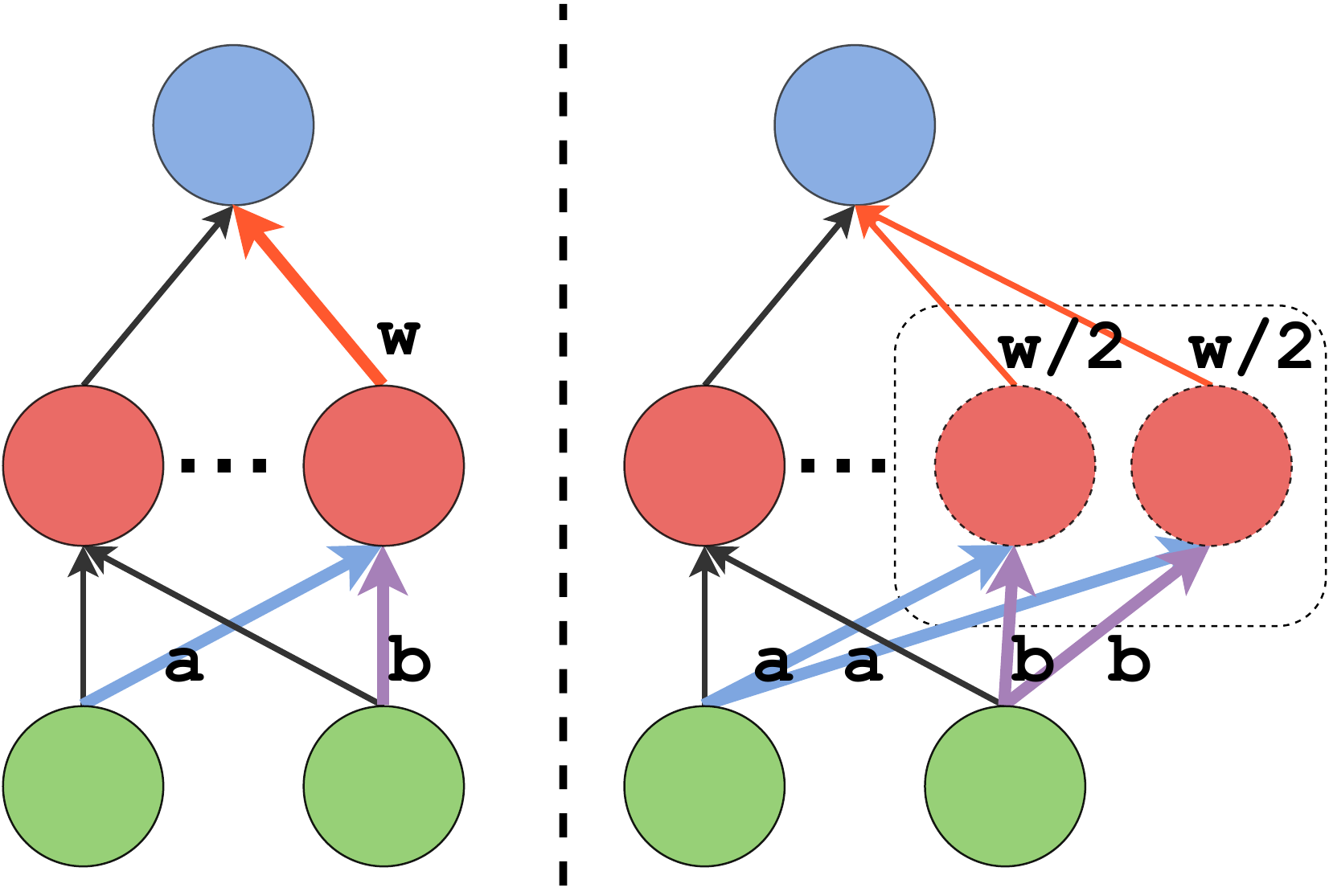}
  \end{center}
  \vspace{-20pt}
  \vspace{-10pt}
\end{wrapfigure}
In particular,  we  split $\theta$ into $m$ off-springs  $\vv\theta := \{\theta_i\}_{i=1}^m$, and replace the neuron $\sigma(\theta ,x)$ with a 
weighted sum of the off-spring neurons $\sum_{i=1}^m w_i \sigma(\theta_i, x)$, where $\vv w:= \{w_i\}_{i=1}^m$  
 is a set of positive weights assigned on the off-springs, and satisfies $\sum_{i=1}^m w_i = 1$, $w_i > 0$.  
This yields an augmented loss function on $\vv\theta$ and $\vv w$: 
\begin{align} \label{equ:LLtheta}
\Lm(\vv\theta, \vv w) :=  
\E_{x\sim \mathcal D}\left [\Phi \left (  \sum_{i=1}^m  w_i\sigma(\para_i, x)   \right ) \right]. 
\end{align}
A key property of this construction is that it introduces a smooth change on the loss function when the off-springs $\{\theta_i\}_{i=1}^m$ are close to the original parameter $\theta$:  
when 
$\theta_i = \theta$, $\forall i=1,\ldots,m$,
the augmented network and loss are equivalent to the original ones, that is, $\Lm(\theta \one_m, \vv w) = L(\theta)$, where $\one_m$ denotes the $m\times 1$ vector consisting of all ones; 
when all the $\{\theta_i\}$ are within an infinitesimal neighborhood of $\theta$, it yields an infinitesimal change on the loss, with which a steepest descent can be derive.  

Formally, 
consider the set of splitting schemes $(m, \vv\theta, \vv w)$ whose off-springs are  
$\epsilon$-close to the original neuron:  
$$
\{(m, \vv\theta,  \vv w)  
\colon 
m \in \mathbb N_+, ~ 
\norm{\theta_i -\theta}\btmptwo\leq \epsilon, ~ \sum_{i=1}^m w_i = 1, ~w_i > 0,~\forall i =1,\ldots, m\}.  
$$
We want to decide the optimal $(m, \vv\theta, \vv w)$ to maximize the decrease of loss $\L(\theta, \vv w)-L(\theta)$, when the step size $\epsilon$ is infinitesimal.  
Although this appears to be an infinite dimensional optimization because $m$ is allowed to be arbitrarily large, 
we show that the optimal choice is achieved with either $m=1$ (no splitting) or $m=2$ (splitting into two off-springs), with uniform weights $w_i = 1/m$.  
Whether a neuron should be split ($m=1$ or $2$)  
and the optimal values of the off-springs $\{\theta_i\}$   
are decided by 
the minimum eigenvalue and eigenvector of a \emph{splitting matrix},  
which plays a role similar to Hessian matrix for deciding saddle points. 
\begin{mydef}[\textbf{Splitting Matrix}]\label{def:splitting}
For $L(\theta)$ in \eqref{equ:Ltheta}, 
its \textbf{splitting matrix} $S(\theta)$ is defined as 
 \begin{align} \label{equ:A}
 S(\theta) =  \E_{x\sim \mathcal D}[\Phi'(\sigma(\theta, x)) \nabla_{\theta\theta}^2 \sigma(\theta,x)  ] . 
 \end{align} 
 We call the minimum eigenvalue $\lambda_{min}(S(\theta))$ of $S(
 \theta)$ the \textbf{splitting index} of $\theta$, 
 and the eigenvector $v_{min}(S(\theta))$ related to 
 $\lambda_{min}(S(\theta))$ the \textbf{splitting gradient} of $\theta$. 
 %
\end{mydef}

 The splitting matrix $S(\theta)$ is a $\RR^{d\times d}$ symmetric ``semi-Hessian'' matrix that involves the first derivative $\Phi'(\cdot)$, and the second  derivative of $\sigma(\theta,x)$.  
 It is useful to compare it with the typical gradient and Hessian matrix of $L(\theta)$: 
 \begin{align*} 
 \nabla_{\theta} L(
 \theta)  =  
 \E_{x\sim \mathcal D}[\Phi'(\sigma(\theta, x)) \nabla_{\theta} \sigma(\theta,x)  ], &~~~~~&
 \nabla_{\theta\theta}^2 
L(\theta)  =  S(\theta) + 
\underbrace{\E[  \Phi''(\sigma(\theta, x)) \nabla_{\theta} \sigma(\theta,x)^{\otimes 2}   ]}_{\text{\small\textcolor{blue}{$T(\theta)$}}}, 
 \end{align*}
 where $v^{\otimes 2} := vv^\top$ is the outer product. 
 The splitting matrix $S(\theta)$ differs from the gradient $\nabla_\theta L(\theta)$ in replacing $\nabla_\theta \sigma(\theta,x)$ with the second-order derivative $\nabla_{\theta\theta}^2\sigma(\theta,x)$, 
 and differs from the Hessian matrix $\nabla_{\theta\theta}^2L(\theta)$ in missing an extra term $T(\theta)$. 
 %
 %
 %
 We should point out that $S(\theta)$
 is the ``easier part'' of 
 the Hessian matrix,  
because the second-order derivative  $\nabla_{\theta\theta}^2\sigma(\theta,x)$ of the individual neuron $\sigma$ is much simpler than the second-order derivative $\Phi''(\cdot)$ of {``everything else''},  
which appears in the extra term $T(\theta)$. 
In addition, as we show in Section~\ref{sec:general}, 
$S(\theta)$ is block diagonal in terms of multiple neurons,  which is crucial for enabling practical computational algorithm. 
 
It is useful to decompose each $\theta_i$ into 
$\theta_i = \theta + \epsilon (\mu + \delta_i)$, where 
$\mu$ is an average displacement vector shared by all copies,  and $\delta_i$ is the splitting vector associated with  $\theta_i$, and satisfies $\sum_i w_i \delta_i = 0$ (which implies $\sum_i w_i \theta_i  = \theta + \epsilon \mu$).  
It turns out that the change of loss $\Lm(\vv \theta,\vv w) -L(\theta)$ naturally decomposes into two terms that reflect the effects of the average displacement and splitting, respectively. 
\begin{thm}\label{thm:Lthetasp}
Assume $\theta_i = \theta + \epsilon (\mu + \delta_i)$  
with $\sum_i w_i\delta_i = 0$ and $\sum_i w_i = 1$.  
For $L(\theta)$ and $\Lm(\vv\theta,\vv w)$ in \eqref{equ:Ltheta} and \eqref{equ:LLtheta}, assume 
$\Lm(\vv\theta,\vv w)$ has bounded third order derivatives w.r.t. $\vv\theta$.  
 We have 
\begin{align}\label{equ:LLLtaylor}
\Lm(\vv\theta, \vv w) - L(\theta) 
& = \underbrace{
\epsilon \nabla L(\theta)^\top \meandelta  
+ \frac{\epsilon^2}{2}\meandelta^\top \nabla^2 L(\theta) \meandelta  }_{
\text{\textcolor{blue}{\small $I(\mu; \theta)= L(\theta + \epsilon \mu) - L(\theta) + \obig(\epsilon^3)$}}} ~+~
\underbrace{\frac{\epsilon^2}{2} \sum_{i=1}^m w_i \delta_i^\top S(\theta)\delta_i}_{\text{\small\textcolor{blue}{$\II\left (\vv\delta, \vv w ; ~\theta\right)$}}}  
~+~ \obig(\epsilon^3), 
\end{align}
where the change of loss is decomposed into two terms: 
    the first term $I(\mu; \theta)$ is the effect of the average displacement $\mu$, and it is equivalent to applying the standard parametric update $\theta\gets \theta + \epsilon \mu$ on $L(\theta)$. 
The second term $\II\left(\vv\delta, \vv w 
;~ \theta\right)$ is the change of the loss caused by the splitting vectors $\vv\delta := \{\delta_i\}$. It depends on $L(\theta)$ only through the splitting matrix $S(\theta)$. 
\end{thm}
%
 %
 %
 %
 Therefore, the optimal average displacement $\mu$ should be  decided by standard parametric steepest (gradient) descent, which yields a typical $\obig(\epsilon)$ decrease of loss at non-stationary points. 
 In comparison, 
 the splitting term $\II(\vv\delta, \vv w; ~ \theta)$ is always  $\obig(\epsilon^2)$, which is much smaller. 
Given that 
introducing new neurons increases model size, 
splitting should not be preferred 
unless it is impossible to achieve  
an $\obig(\epsilon^2)$ gain with pure parametric updates that do not increase the model size. 
%
Therefore, it is  motivated to introduce splitting only at stable local minima, 
when the optimal $\mu$ equals zero and no further improvement is possible with (infinitesimal) regular parametric descent on $L(\theta)$. In this case, we only need to minimize the splitting term $\II(\vv\delta,\vv w; \theta)$ to decide the optimal splitting strategy, which is shown in the following theorem. 

\begin{thm}\label{thm:opt}
%
%
a) If the splitting matrix is positive definite, that is, 
$\lambda_{min}(S(\theta))>0$, we have $\II(\vv \delta, \vv w; \theta)>0$ for any $\vv w > 0$ and $\vv\delta\neq 0$, and hence no infinitesimal splitting can decrease the loss. We call that $\theta$ is {splitting stable} in this case. 

b) If $\lambda_{min}(S(\theta)) <0$, 
an optimal splitting strategy that minimizes $\II(\vv\delta, \vv w; \theta)$ subject to $\norm{\delta_i} \leq 1$ is 
\begin{align*} 
m =2, &&
w_1 = w_2 = 1/2, &&
\text{and}  &&
\delta_1  = v_{min}(S(\theta)), && 
\delta_2  = - v_{min}(S(\theta)), 
\end{align*}
where $v_{min}(S(\theta))$, called the {splitting gradient}, is the eigenvector related to  $\lambda_{min}(S(\theta))$.  
Here we split the neuron into two copies of equal weights, and update each copy with the splitting gradient. 
The change of loss obtained in this case is $\II(\{\delta_1, -\delta_1\}, \{1/2,1/2\}; ~\theta)= -{\epsilon^2}\lambda_{min}(S(\theta))/2 < 0$. 
\end{thm}

\paragraph{Remark} 
The splitting stability ($S(\theta)\succ 0$) does not necessarily ensure 
the standard parametric stability of $L(\theta)$ (i.e., $\nabla^2 L(\theta)= S(\theta) + T(\theta) \succ 0$), 
except when $\Phi(\cdot)$ is convex which ensures $T(\theta) \succeq 0$ (see Definition~\ref{def:splitting}). 
If both $S(\theta) \succ0$ and $\nabla^2 L(\theta) \succ 0$ hold, the loss can not be improved by any 
local update or splitting, no matter how many off-springs are allowed. 
Since stochastic gradient descent guarantees to escape unstable stationary points \citep{lee2016gradient, jin2017escape},  
we only need to calculate $S(\theta)$ to decide the splitting stability in practice.

\subsection{Splitting Deep Neural Networks} 
\label{sec:general}
In practice, we need to split multiple neurons simultaneously, 
which may be of different types, or  locate in different layers of a deep neural network. 
The key questions 
are 
if the optimal splitting strategies of different neurons influence each other in some way, 
and how to compare the gain of splitting different neurons and select the best subset of neurons to split under a budget constraint.  

It turns out the answers are simple.  
We show that 
the change of loss caused by splitting a set of neurons 
is simply the sum of the splitting terms $\II(\vv\delta, \vv w; \theta)$ of the individual neurons. 
Therefore, we can calculate the splitting matrix of each neuron independently without considering the other neurons,
and compare the ``splitting desirability'' of the different neurons by their minimum eigenvalues (splitting indexes).  
This motivates our main algorithm ({Algorithm \ref{alg:main}}), 
in which 
 we progressively split the neurons with the most negative splitting indexes following their own splitting gradients.  
 Since the neurons can be in different layers and of different types, this provides an adaptive way to grow neural network structures to fit best with data.  

\begin{algorithm*}[t] 
\caption{Splitting Steepest Descent for Optimizing Neural Architectures}
\begin{algorithmic} 
    \STATE \textbf{Initialize} a neural network with a set of neurons $\theta^{[1:n]}=\{\theta^{[\ell]}\}_{\ell=1}^n$ that can be split, whose loss satisfies \eqref{equ:lell}. 
    Decide a maximum number $m_{*}$ of neurons to split at each iteration, and a threshold $\lambda_* \leq 0$ of the splitting index.  A stepsize $\epsilon$. 
    \STATE \textbf{1. Update the parameters} using standard optimizers (e.g., stochastic gradient descent) until no further improvement can be made by only updating parameters. 
    \vspace{.3\baselineskip} 
    \STATE \textbf{2. Calculate the splitting matrices} $\{S^{[\ell]}\}$ of the neurons following \eqref{equ:sell}, as well as their minimum eigenvalues $\{\lambda^{[\ell]}_{min}\}$ and the associated eigenvectors $\{v^{[\ell]}_{min}\}$.  
    \vspace{.3\baselineskip} 
    \STATE 
    \textbf{3. Select the set of neurons to split} by 
    picking the top $m_*$ neurons with the smallest eigenvalues $\{\lambda^{[\ell]}_{min}\}$ and satisfies $\lambda_{min}^{[\ell]} \leq \lambda_*$.   
   \vspace{.3\baselineskip}     
    \STATE \textbf{4. Split each of the selected neurons} into two off-springs with equal weights, and update the neuron network by replacing each selected neuron  $\sigma_\ell(\theta^{[\ell]}, ~\cdot)$ with 
    \begin{align*}
    \frac{1}{2}(\sigma_\ell(\theta_1^{[\ell]}, ~\cdot) + \sigma_\ell(\theta_2^{[\ell]}, ~\cdot)),  
    &&
    \text{where} 
    &&
    \theta_1^{[\ell]} \gets \theta^{[\ell]} + \epsilon v_{min}^{[\ell]}, 
    && 
    \theta_2^{[\ell]} \gets \theta^{[\ell]} - \epsilon v_{min}^{[\ell]}. 
    \end{align*}
    Update the list of neurons. Go back to Step 1 or stop when a stopping criterion is met. 
\end{algorithmic}
\label{alg:main}  
\end{algorithm*}

To set up the notation, 
let $\theta^{[1:n]}=\{\theta^{[1]},\ldots \theta^{[n]}\}$ 
be the parameters of a set of  neurons (or any duplicable sub-structures) in a large neural network, 
where $\theta^{[\ell]}$ is the parameter of the $\ell$-th neuron. 
Assume we split $\theta^{[\ell]}$ into $m_\ell$ copies 
$\vv\theta^{[\ell]} := \{\theta_i^{[\ell]}\}_{i=1}^{m_\ell}$, with weights $\vv w^{[\ell]} = \{w_i^{[\ell]}\}_{i=1}^{m_\ell}$ satisfying $\sum_{i=1}^{m_\ell} w_i^{[\ell]} = 1$ and $w_i^{[\ell]}\geq0$, $\forall i=1,\ldots, m_\ell$. 
Denote by $L(\theta^{[1:n]})$ and $\L(\vv\theta^{[1:n]}, ~ \vv w^{[1:n]})$ the loss function of the original and augmented networks, respectively. 
It is hard to specify the actual expression of the loss functions in general cases,  but it is sufficient to know that $L(\theta^{[1:n]})$  depends on  each $\theta^{[\ell]}$ only through the output of its related neuron, 
\begin{align} \label{equ:lell}
L(\theta^{[1:n]})   
= \E_{x\sim \mathcal D}\left [\Phi_\ell 
\left (
\sigma_\ell \left (\theta^\supell,~~ h^\supell \right ); ~~  \theta^{[\neg \ell]} \right )\right ], ~~~~~~~~~~~h^\supell = g_\ell(x; ~~ \theta^{[\neg \ell]}), 
\end{align}
where $\sigma_\ell$ denotes the activation function of neuron $\ell$, 
and $g_\ell$ and $\Phi_\ell$ denote the parts of the loss that connect to the input and output of neuron $\ell$, respectively, 
both of which depend on the other parameters $\theta^{[\neg \ell]}$  in some complex way. 
Similarly, the augmented loss $\L(\vv\theta^{[1:n]}, ~ \vv w^{[1:n]})$ satisfies 
\begin{align} \label{equ:LLLLL}
\L(\vv\theta^{[1:n]}, \vv w^{[1:n]}) 
= \E_{x\sim \mathcal D}\left [\vv\Phi_\ell 
\left ( \sum_{i=1}^{m_\ell} w_i 
\sigma_\ell \left (\theta_i^\supell,~~  \vv h^\supell \right ); ~~ \vv \theta^{[\neg \ell]}, \vv w^{[\neg \ell]} \right )\right ], 
\end{align}
where $\vv h^\supell = \vv{g}_\ell(x; ~~\vv \theta^{[\neg \ell]}, \vv w^{[\neg \ell]})$, and 
$\vv g_\ell$, $\vv \Phi_\ell$ are the augmented variants of $g_\ell$, $\Phi_\ell$, respectively. 

Interestingly, although each equation in \eqref{equ:lell} and \eqref{equ:LLLLL} only provides a partial specification of the loss function of deep neural nets, 
they together are sufficient 
to 
establish the following key extension of  
Theorem~\ref{thm:Lthetasp} to the case of multiple neurons. 
\begin{thm}\label{thm:general} 
Under the setting above, 
assume $\theta_i^\supell = \theta^\supell + \epsilon(\mu^\supell~+~\delta_i^\supell)$ for $\forall \ell\in[1\!:\!n]$,  
where $\mu^{[\ell]}$ denotes the average displacement vector on $\theta^{[\ell]}$, 
and $\delta_i^\supell$ is the $i$-th splitting vector of $\theta^{[\ell]}$, with  $\sum_{i=1}^{m_\ell} w_i \delta_i^\supell = 0$. {Assume $\L(\vv\theta^{[1:n]}, \vv w^{[1:n]})$ has bounded third order derivatives w.r.t. $\vv\theta^{[1:n]}$.}    
We have 
$$\L(\vv\theta^{[1:n]}, \vv w^{[1:n]})  = L(\theta^{[1:n]} + \epsilon \mu^{[1:n]}) + \sum_{\ell=1}^n  
\underbrace{\frac{\epsilon^2}{2} \sum_{i=1}^{m_\ell} w_i^\supell  {\delta_i^\supell}^\top  S^\supell(\theta^{[1:n]}) \delta_i^\supell}_{ 
\textcolor{blue}{\text{\small $\II_\ell(\vv\delta^\supell, \vv w^\supell; ~ \theta^{[1:n]})$ }}}  +\obig(\epsilon^3), $$
where the effect of average displacement is again equivalent to that of the corresponding parametric update $\theta^{[1:n]}\gets \theta^{[1:n]} + \epsilon \mu^{[1:n]}$;  
the splitting effect equals the sum of the individual splitting terms 
$\II_\ell(\vv\delta^\supell, \vv w^\supell; ~ \theta^{[1:n]})$, which depends on the splitting matrix  
 $S^\supell(\theta^{[1:n]})$ 
of neuron $\ell$, 
\begin{align} \label{equ:sell}
S^\supell(\theta^{[1:n]}) = 
\E_{x\sim \mathcal D}\left [{\nabla_{\sigma_\ell}}\Phi_\ell
\left (\sigma_\ell\left (\theta^\supell,~~ h^\supell \right ); ~~
\theta^{[\neg \ell]} 
\right ) \nabla^2_{\theta\theta}  \sigma_\ell\left (\theta^\supell,~~ 
h^\supell
\right ) 
\right ]. 
\end{align}
\end{thm}
The important implication of 
Theorem~\ref{thm:general} 
is that there is \emph{no crossing term} in the splitting matrix, 
unlike the standard Hessian matrix. 
Therefore, the splitting effect of an individual neuron 
only depends on its own splitting matrix and can be evaluated individually; 
the splitting effects of different neurons can be compared using their splitting indexes,  
allowing us to decide the best subset of neurons to split when a maximum number constraint is imposed. 
As shown in Algorithm~\ref{alg:main}, 
 we decide a maximum number $m_{*}$ of neurons to split at each iteration, 
 and a threshold $\lambda_* \leq 0$ of splitting index, 
 and split the neurons whose splitting indexes are ranked in top $m_*$ and smaller than $\lambda_*$. 

\paragraph{Computational Efficiency} 
The computational cost of exactly evaluating all the splitting indexes and gradients  
 on a data instance is $\obig(n d^3)$, where $n$ is the number of neurons and $d$ is the number of 
the parameters of each neuron. 
Note that this is much better than evaluating  
the Hessian matrix, which costs $\obig(N^3)$, where $N$ is the total number of parameters (e.g., $N \geq nd$). 
In practice, $d$ is not excessively large or can be controlled by identifying a subset of important neurons to split. 
Further computational speedup can be 
obtained by using efficient gradient-based large scale eigen-computation methods, which we investigate in future work.  
\myempty{[
For very large scale settings, 
we can approximate the eigenvalues and eigenvectors efficiently 
by gradient descent on Rayleigh quotient, which only requires to evaluate the matrix-vector product $S(\theta)v$ via $\nabla_y F_2(0)$ without expanding the whole splitting matrix, yielding a low computational cost similar to that of the standard parametric gradient descent (See Appendix~{\ref{sec:fast_grad_approx}} for more discussion). ]
}

\subsection{Splitting as $\Linfty$-Wasserstein Steepest Descent}  
\label{sec:wasser} 
We present a functional aspect of our approach,
in which we frame the co-optimization of the neural parameters and structures into a functional optimization in the space of distributions of the neuron weights, and
show that our splitting strategy can be viewed as
a second-order descent for escaping saddle points  
in the $\Linfty$-Wasserstein space of distributions, 
while the standard parametric gradient descent corresponds to a first-order descent in the same space. 

We illustrate our theory using the single neuron case in Section~\ref{sec:single}. 
Consider the augmented loss $\L(\vv\theta, \vv w)$ in \eqref{equ:LLtheta}.  
Because the off-springs of the neuron are exchangeable,  we can equivalently represent $\L(\vv\theta,\vv w)$ as a functional of the empirical measure of the off-springs, 
\begin{align}\label{equ:Lrho}
\L[\rho] = \E_{x\sim \mathcal D}\left [ \Phi\left (  \E_{\theta\sim\rho}[\sigma(\theta, x) ] \right ) \right], 
&&
\rho =\sum_{i=1}^{m}w_i \delta_{\theta_i}, 
\end{align} 
where $\delta_{\theta_i}$ denotes the delta measure on $\theta_i$ and $\L[\rho]$  is the functional representation of $\L(\vv\theta, \vv w)$. 
The idea is to optimize $\L[\rho]$ 
in the space of probability distributions (or measures) using a functional steepest descent. 
To do so, a notion of distance on the space of distributions  need to be decided. We consider the $p$-Wasserstein metric, 
\begin{align}\label{equ:dinfty}
\D_p(\rho, \rho') =   \inf_{\gamma\in\Pi(\rho,\rho')} \left(\E_{(\theta,\theta')\sim \gamma}[\norm{\theta-\theta'}^p] \right)^{1/p}, &&\text{for $p > 0$},
\end{align}
where $\Pi(\rho,\rho')$ denotes the set of probability measures whose first and second marginals are $\rho$ and $\rho'$, respectively, 
 and $\gamma$ can be viewed as describing a transport plan from $\rho$ to $\rho'$. 
We obtain the $\infty$-Wasserstein metric $\D_\infty(\rho, \rho')$ in the limit when $p\to+\infty$, in which case the $p$-norm reduces to an esssup norm, that is, 
$$\D_\infty(\rho, \rho') =   \inf_{\gamma\in\Pi(\rho,\rho')} \esssup_{(\theta,\theta')
\sim \gamma}[\norm{\theta-\theta'}],$$
where the $\esssup$ notation denotes the smallest number $c$ such that the set $\{(\theta,\theta')\colon~ \norm{\theta -  \theta'}\btmptwo>c\}$ has zero probability under $\gamma$.  
See more discussion in \citet{villani2008optimal} and Appendix~\ref{sec:wassappendix}. 

The $\Linfty$-Wasserstein metric yields a natural connection to node splitting.  
For each $\theta$,  
the conditional distribution $\gamma(\theta'~|~\theta)$  represents the distribution of points $\theta'$  transported from $\theta$, which can be viewed as the off-springs of $\theta$ in the context of node splitting. 
If $\D_\infty(
\rho,\rho')\leq \epsilon$, 
it means that $\rho'$ can be obtained from splitting $\theta\sim \rho$ 
such that all the off-springs are $\epsilon$-close, i.e., $\norm{\theta' -\theta}\btmptwo \leq \epsilon$.  
This is consistent with the augmented neighborhood introduced in Section~\ref{sec:single}, 
except that $\gamma$ here can be an absolutely continuous distribution, representing a continuously infinite number of off-springs; 
but this yields no practical difference because any distribution $\gamma$ can be approximated arbitrarily close using a countable number of particles.
Note that $p$-Wasserstein metrics with finite $p$ are not suitable for our purpose because $\D_p(\rho,\rho')\leq \epsilon$ with $p<\infty$  
does not ensure $\norm{\theta'-\theta}\leq \epsilon$ for all $\theta\sim \rho$ and $\theta'\sim \rho'$. 

Similar to the steepest descent on the Euclidean space, 
the $\Linfty$-Wasserstein steepest descent on $\L[\rho]$ should iteratively 
find new points that maximize the decrease of loss in an $\epsilon$-ball of the current points. 
Define 
\begin{align*} 
\rho^{*}=\argmin_{\rho'} \{\L[\rho']-\L[\rho] \colon~~~ \D_\infty(\rho,\rho')\leq \epsilon  \},
&& 
\Delta^*(\rho, \epsilon) =\L[\rho^{*}]-\L[\rho]. 
\end{align*}
We are ready to show the connection of Algorithm~\ref{alg:main} to the $\Linfty$-Wasserstein steepest descent. 
\begin{thm}\label{thm:wassdescent}
Consider the $\L(\vv\theta,\vv w)$ and $\L[\rho]$ in \eqref{equ:LLtheta} and \eqref{equ:Lrho}, connected with $\rho = \sum_i w_i \delta_{\theta_i}$. 
Define $G_\rho(\theta) = \E_{x\sim \mathcal D} \left [ \Phi'(f_\rho(x)) \nabla_{\theta } \sigma(\theta,x) \right ]$ and 
$S_\rho(\theta) = \E_{x\sim \mathcal D} \left [ \Phi'(f_\rho(x)) \nabla_{\theta\theta}^2\sigma(\theta,x) \right ]$ 
with $f_\rho(x)=\E_{\theta\sim\rho}[\sigma(\theta,x)]$,  
which are related to the gradient and splitting matrices of $\L(\vv\theta,\vv w)$, respectively. 
Assume $\L(\vv\theta,\vv w)$ has bounded third order derivatives w.r.t. $\vv\theta$. 

a) If $\L(\vv \theta,\vv w)$ is on a non-stationary point w.r.t. $\vv \theta$, 
then the steepest descent of $\L[\rho]$ is achieved by moving
all the particles of $\rho$ with gradient descent on $\L(\vv\theta,\vv w)$, that is, 
$$
\L[(I-\epsilon G_\rho) \sharp \rho] - \L[\rho] = \Delta^*(\rho, \epsilon) + \obig(\epsilon^2) 
= - \epsilon\E_{\theta\sim\rho}[\norm{G_\rho(\theta)}] + \obig(\epsilon^2), 
$$
where $(I-\epsilon G_\rho) \sharp \rho$ denotes 
the distribution of $\theta' = \theta  - \epsilon  G_\rho(\theta)/\norm{G_\rho(\theta)}$ when $\theta\sim \rho$. 

b) 
If  $\L(\vv \theta,\vv w)$ reaches a stable local optima w.r.t. $\vv\theta$, 
the steepest descent on $\L[\rho]$ is splitting each neuron with $\lambda_{min}(S_\rho(\theta)) <0$ into two copies of equal weights following their minimum eigenvectors, while keeping the remaining neurons to be unchanged.    
Precisely, denote by $(I\pm \epsilon v_{min}(S_\rho(\theta))_+) \sharp \rho$ the distribution obtained in this way, we have 
$$
\L[(I\pm \epsilon v_{min}(S_\rho(\theta))_+) \sharp \rho] - \L[\rho] = \Delta^*(\rho, \epsilon) + \obig(\epsilon^3) , 
$$
where we have $\Delta^*(\rho, \epsilon) = \epsilon^2  \E_{\theta\sim\rho}[\min(\lambda_{min}(S_\rho(\theta)), 0)]/2$.  
\end{thm}
\paragraph{Remark}  
There has been a line of theoretical works on analyzing gradient-based learning of neural networks via $2$-Wasserstein gradient flow by considering the \emph{mean field limit}  when the number of neurons $m$ goes to infinite  $(m\to\infty)$ \citep[e.g.,][]{mei2018mean, chizat2018global}. 
These analysis focus on the first-order descent on the $2$-Wasserstein space as a theoretical tool for understanding the behavior of gradient descent on overparameterized neural networks. 
Our framework is significant different, since we mainly consider the second-order descent on the $\infty$-Wasserstein space, and  the case of finite number of neurons $m$ in order to derive practical algorithms.  
\section{Experiments}
\label{sec:exp}
We test our method on both toy and realistic tasks, 
including learning interpretable neural networks, 
architecture search for  
image classification and energy-efficient keyword spotting. 
Due to limited space, 
many of the detailed settings are shown in Appendix,
in which we also include additional results on 
distribution approximation (Appendix~\ref{sec:toy_mmd}).

\begin{figure*}[ht]
\centering
\setlength{\tabcolsep}{0.3pt}
\hspace{-0.8em}
\begin{tabular}{cccc}
\includegraphics[height=0.16\textwidth]{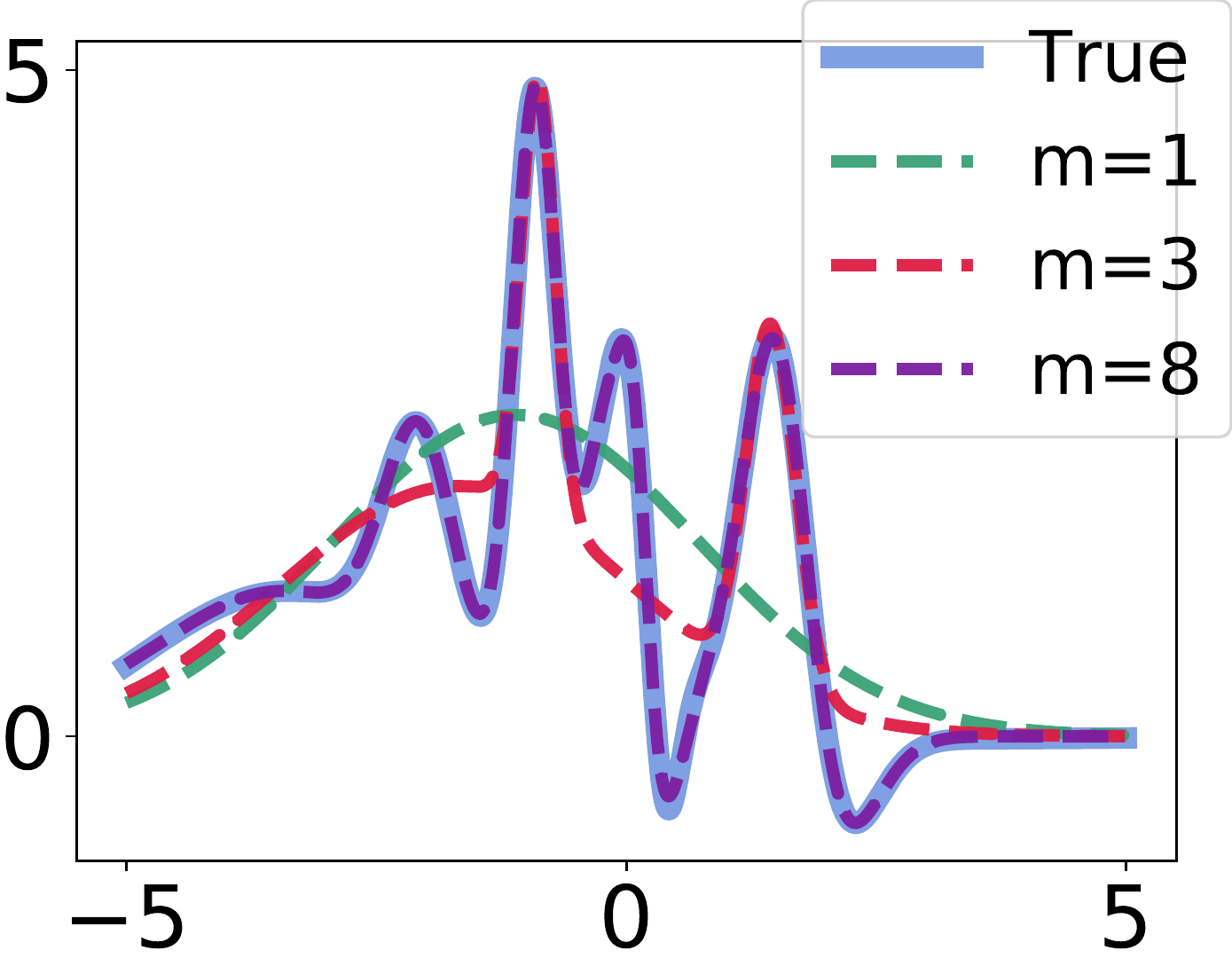} &
\raisebox{1.5em}{\rotatebox{90}{\scriptsize Eigenvalues}}
\includegraphics[height =0.17\textwidth]{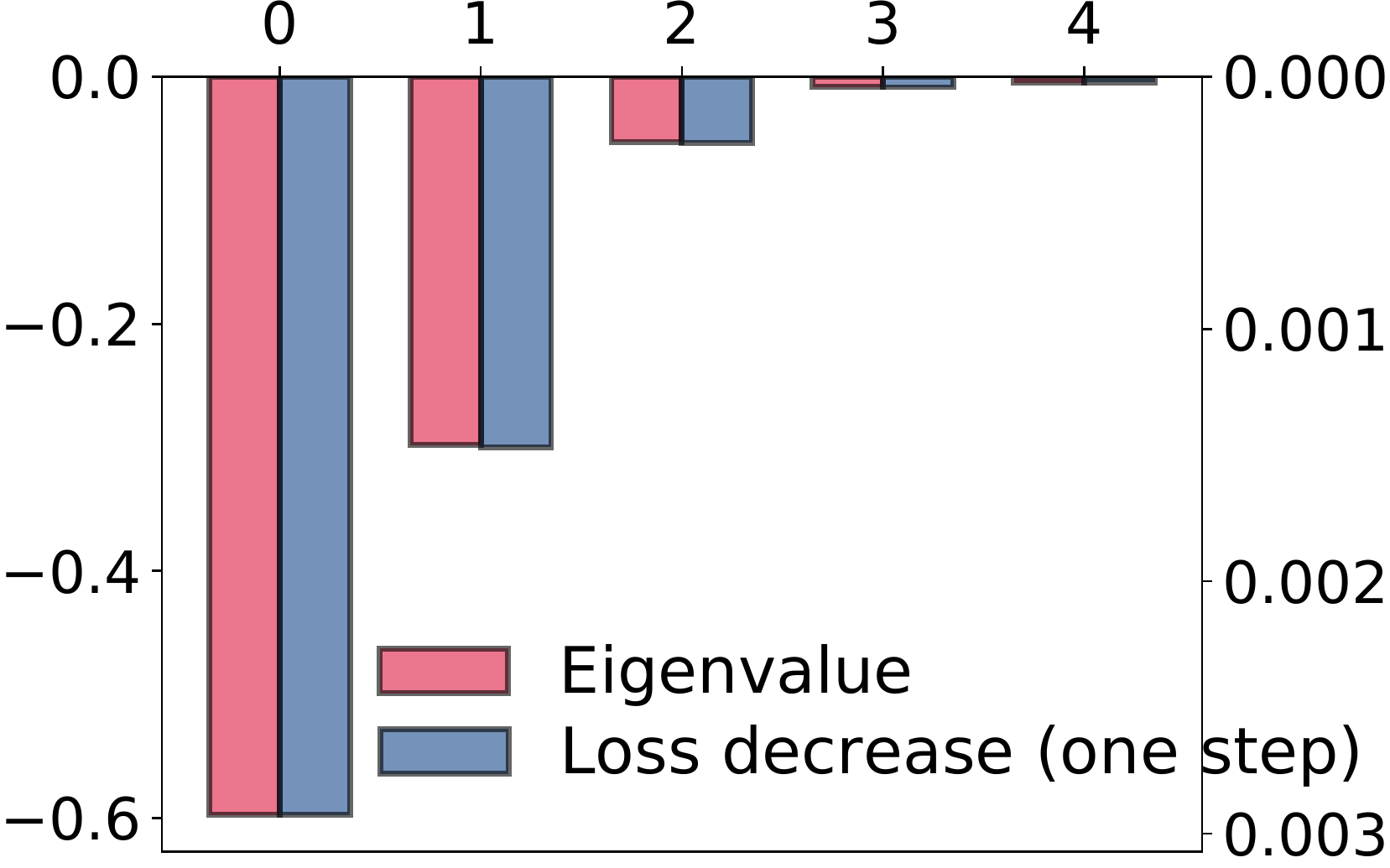}
\raisebox{1.0em}{\rotatebox{90}{\scriptsize Loss decrease }}~ &
\includegraphics[height =0.16\textwidth]{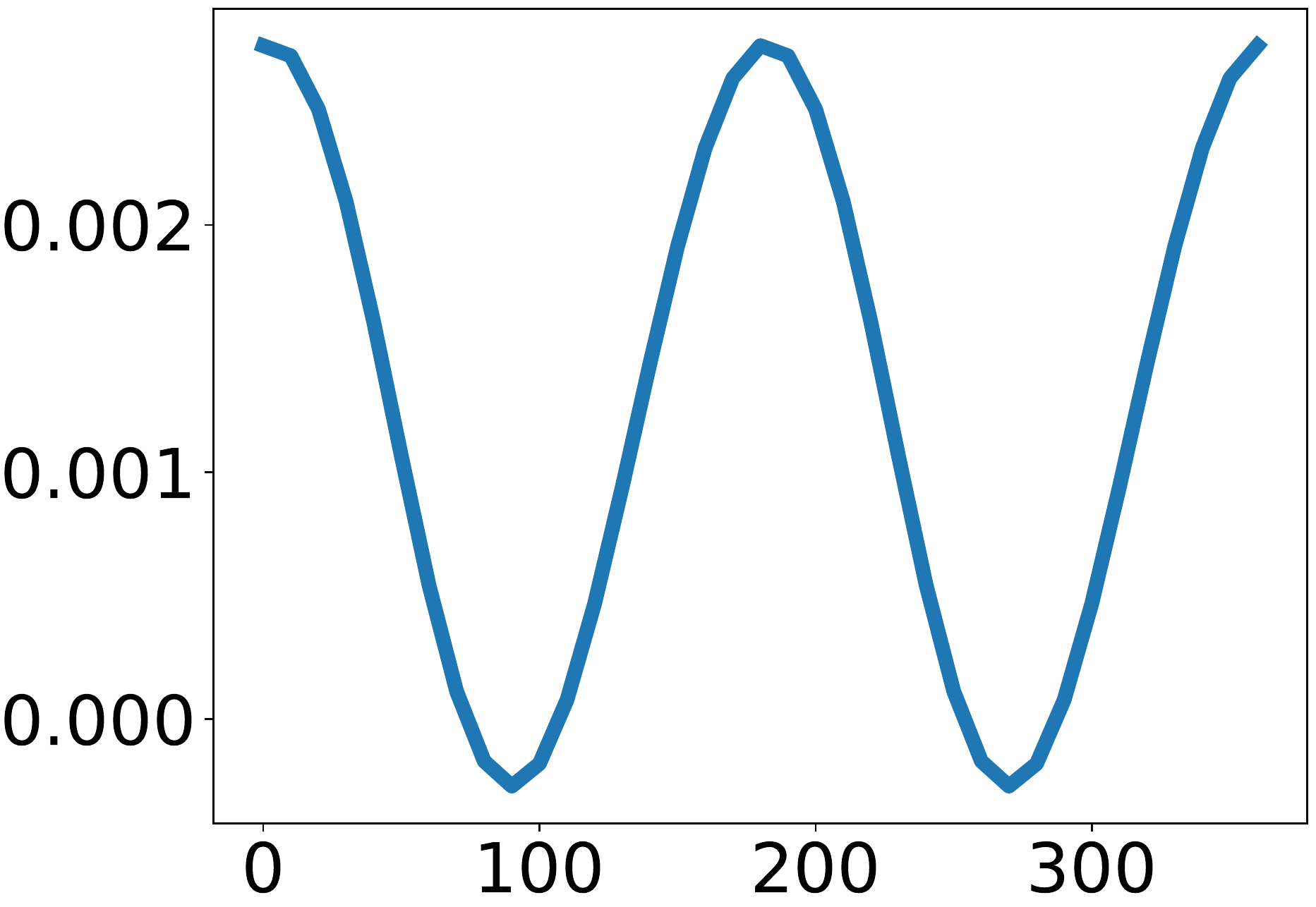}\vspace{0.1em} & 
~\raisebox{1.5em}{\rotatebox{90}{\scriptsize Training Loss}} \!\!\!
\includegraphics[height =0.17\textwidth]{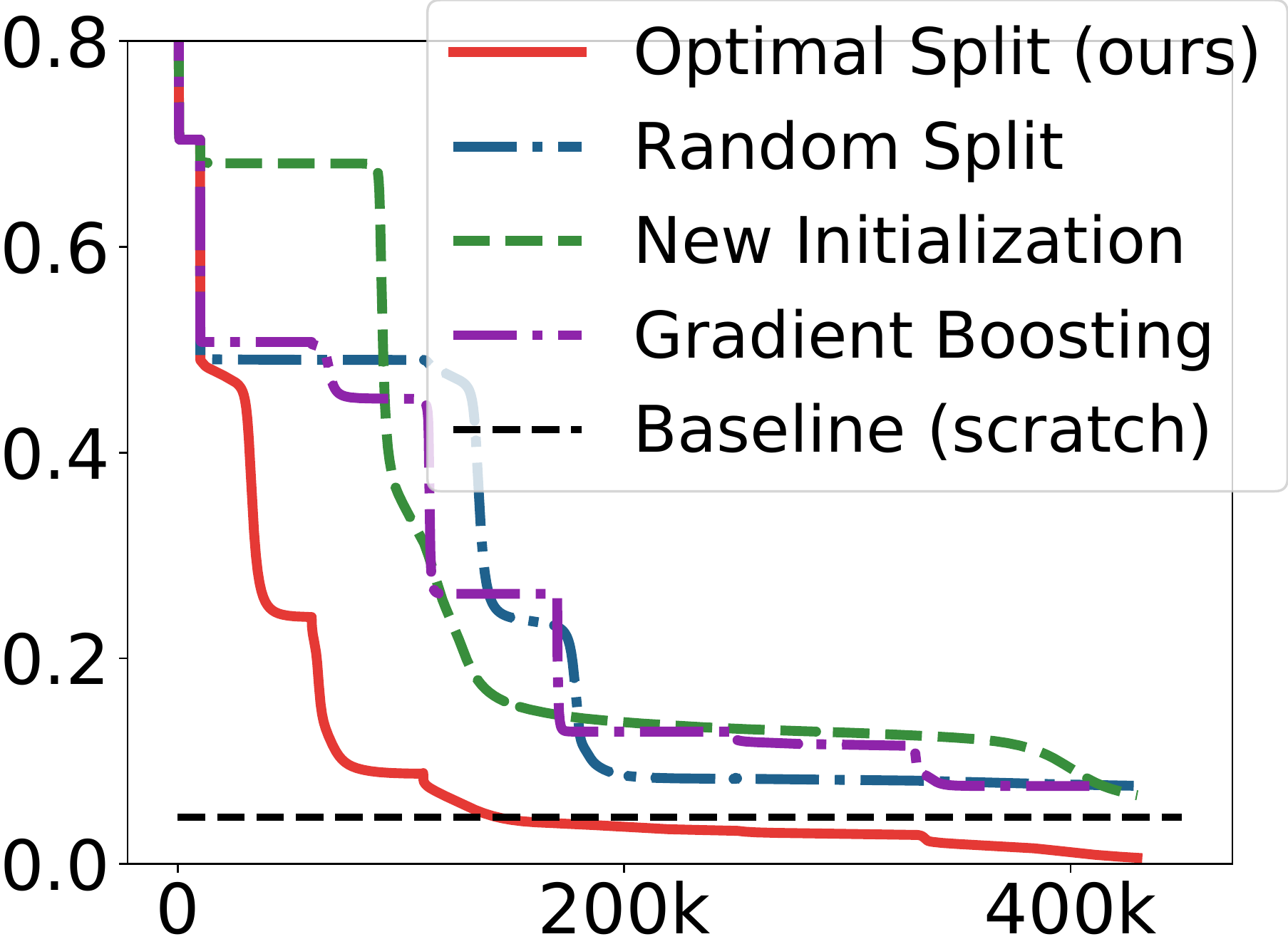} 
\vspace{-0.5em} \\
\scriptsize (a) $x$& \scriptsize (b) &\scriptsize{ \qquad (c) Angle}& {
\scriptsize (d) \#Iteration } \\
\end{tabular}
\caption{\small Results on a one-dimensional RBF network. 
(a) The true and estimated functions. (b) The eigenvalue vs. loss decrease. 
(c) The loss decrease vs. the angle of the splitting direction with the minimum eigenvector. (d) The training loss vs. the iteration (of gradient descent); the splittings happen at the cliff points.  
}
\label{fig:nn_toy}
\end{figure*}
\vspace{-0.75em}

\paragraph{Toy RBF Neural Networks}   
We apply our method to learn a one-dimensional RBF  neural network shown in Figure~\ref{fig:nn_toy}a. 
See Appendix~\ref{sec:app_two_layer_rbf_nn} for details of the setting. 
%
We start with a small neural network with $m=1$ neuron 
and gradually increase the model size by splitting neurons.
Figure~\ref{fig:nn_toy}a shows that we almost recover the true function as we split up to $m=8$ neurons. 
Figure~\ref{fig:nn_toy}b shows the top five eigenvalues and the decrease of loss 
when we split   $ m = 7$ neurons to  $m = 8$ neurons;  
we can see that the eigenvalue and loss decrease correlate linearly, confirming our results in Theorem~\ref{thm:general}. 
Figure~\ref{fig:nn_toy}c 
shows the decrease of the loss 
when we split the top one neuron  following the direction 
with different angles from the minimum eigenvector at $m=7$.  
We can see that the decrease of the loss is maximized 
when the splitting direction aligns with the eigenvector, consistent with our theory. 
In Figure~\ref{fig:nn_toy}d, 
we compare with different baselines of progressive training, including
\texttt{Random Split}, splitting a randomly chosen neuron with a random direction; 
\texttt{New Initialization}, adding a new neuron with randomly initialized weights and co-optimization it with previous neurons; 
\texttt{Gradient Boosting}, 
adding new neurons with Frank-Wolfe algorithm while fixing the previous neurons; 
\texttt{Baseline (scratch)}, training a  network of size $m=8$ from scratch. 
Figure~\ref{fig:nn_toy}d shows our method yields the best result. 

\paragraph{Learning Interpretable Neural Networks} 
To visualize the dynamics of the splitting process, 
we apply our method to incrementally train 
an interpretable neural network designed by  \citet{li2018deep},  
which contains a ``prototype layer'' whose weights are enforced to be similar to realistic images to encourage interpretablity.  
See Appendix~\ref{sec:app_explainable_nn} and \citet{li2018deep} for more detailed settings. 
We apply our method to split the prototype layer starting from a single neuron on MNIST, 
and show 
in Figure~\ref{fig:mnistvisual} 
the evolutionary tree of the neurons in our splitting process. 
We can see that the blurry (and hence less interpretable) prototypes tend to be selected and split into two off-springs that are similar yet more interpretable. 
Figure~\ref{fig:mnistvisual} (b) shows the 
decrease of loss  
when we split 
each of the five neurons at the 5-th step 
(with the decrease of loss measured at the local optima reached  dafter splitting); 
we find that the eigenvalue correlates well 
with the decrease of loss  and the interpretablity of the neurons. 
The complete evolutionary tree and quantitative comparison 
with baselines are shown in Appendix~{\ref{sec:app_explainable_nn}}. 

\begin{figure*}[ht]
\centering
\setlength{\tabcolsep}{0.0pt}
\scalebox{0.95}{
\begin{tabular}{cc}
\includegraphics[height =0.27\textwidth]{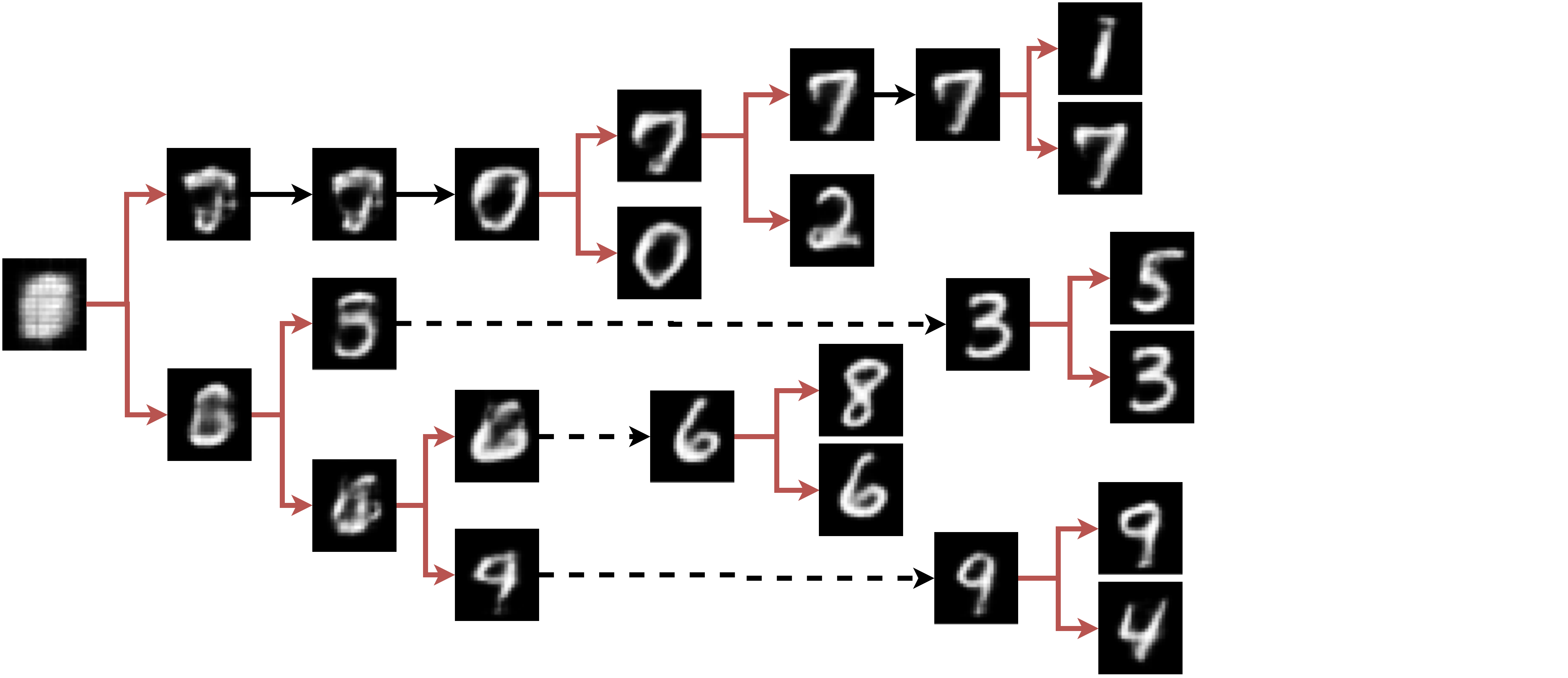}
&
\hspace{-5.2em}
\raisebox{1.2em}{\rotatebox{90}{\scriptsize Eigenvalues}}
\includegraphics[height =0.27\textwidth]{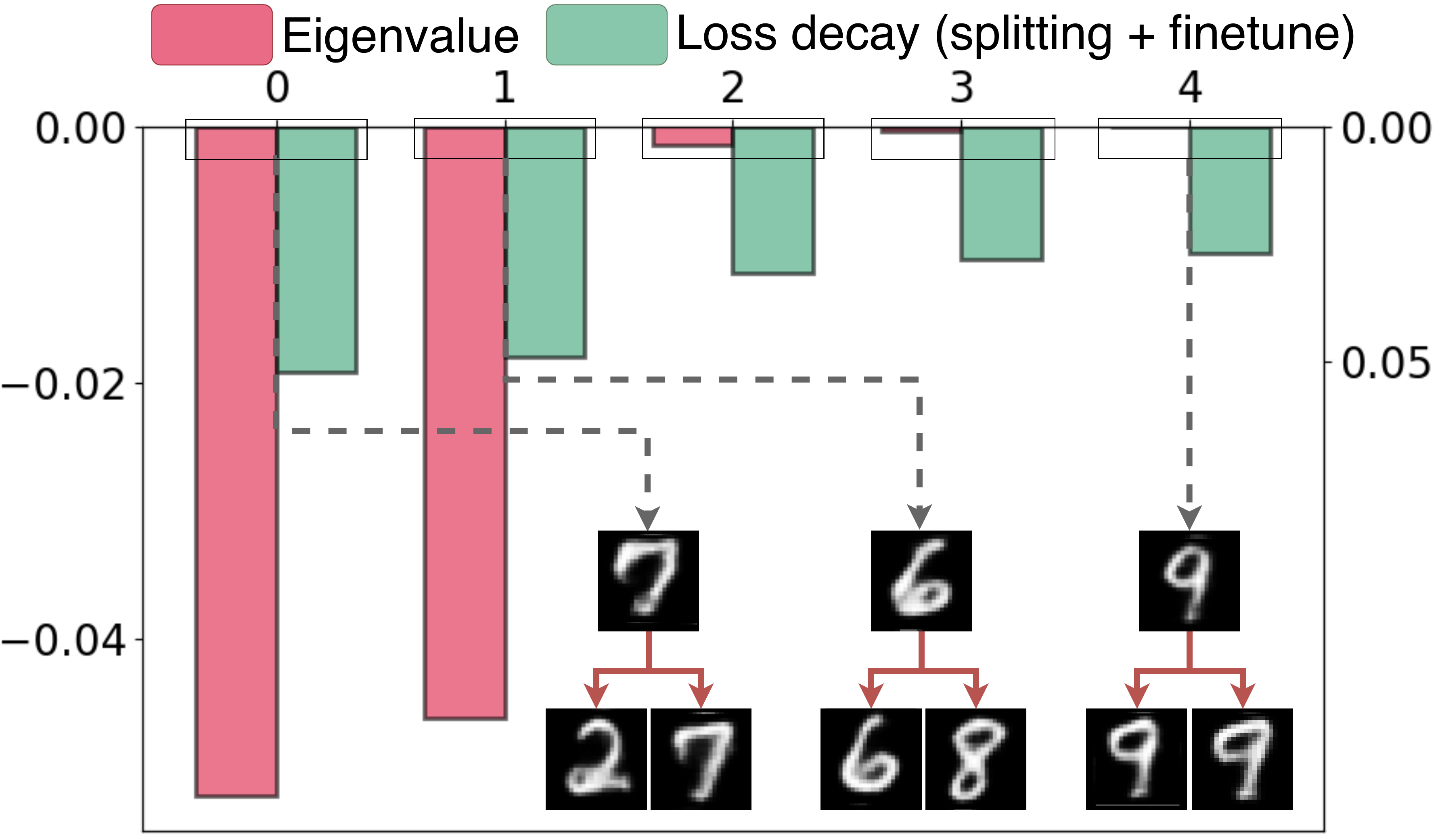} 
\hspace{-1.2em}
\raisebox{1.2em}{\rotatebox{90}{\scriptsize Loss decrease}}\\
{\scriptsize (a)} & {\scriptsize (b)} \vspace{-0.2em}\\
\end{tabular}
}
\caption{\small 
Progressive learning of the interpretable prototype network in \citet{li2018deep} on MNIST. 
(a) The evolutionary tree of our splitting process, in which the least interpretable, or most ambiguous prototypes tend to be split first. 
(b) The eigenvalue and resulting loss decay when splitting the different neurons at the 5-th step. 
}
\label{fig:mnistvisual}
\end{figure*}
\vspace{-0.5em}

\paragraph{Lightweight Neural  Architectures for Image Classification} 
We investigate the effectiveness of our methods 
in learning small and efficient network structures for image classification. 
We experiment with two popular deep neural architectures,  MobileNet~\citep{howard2017mobilenets} and  VGG19~\citep{simonyan2014very}. 
In both cases, 
we start with a relatively small network 
and gradually grow the network by splitting the convolution filters 
following Algorithm~\ref{alg:main}. 
 See Appendix~\ref{sec:image_classfication} for more details of the setting. 
Because there is no other off-the-shelf progressive growing algorithm  
that can adaptively decide the neural architectures like our method,  
we compare with pruning methods, 
which follow the opposite direction of gradually removing neurons starting from a large pre-trained network. 
We test two  state-of-the-art pruning methods, including batch-normalization-based pruning (Bn-prune) \citep{liu2017learning} and L1-based pruning (L1-prune) \citep{li2016pruning}. 
As shown in Figure~\ref{fig:cifar10}a-b, 
our splitting method 
yields higher accuracy with similar model sizes. This is surprising and significant, because the pruning methods leverage the knowledge from a large pre-train model, while our method does not. 

To further test the effect of architecture learning 
in both splitting and pruning methods, 
we 
test another setting in which 
we discard the weights of the neurons 
and retain the whole network starting from a random initialization, under the structure obtained from splitting or pruning at each iteration.  
As shown in Figure~\ref{fig:cifar10}c-d, the results of retraining 
is comparable with (or better than) the result of successive finetuning in Figure~\ref{fig:cifar10}a-b, which is consistent with the findings in \citet{liu2018rethinking}. Meanwhile, our splitting method still outperforms 
both Bn-prune and  L1-prune. 
%
%
%

\begin{figure*}[t]
\centering
\setlength{\tabcolsep}{0.5pt}
\vspace{-0.5em}
\scalebox{0.98}{
\begin{tabular}{cccc}
{\small ~~~MobileNet (finetune)} & 
{\small ~~~VGG19 (finetune)} & 
{\small ~~~MobileNet (retrain)} & 
{\small VGG19 (retrain)} \\
\hspace{-1em}
\raisebox{1.4em}{\rotatebox{90}{\scriptsize Test Accuracy}}~~
\includegraphics[height =0.22\textwidth]{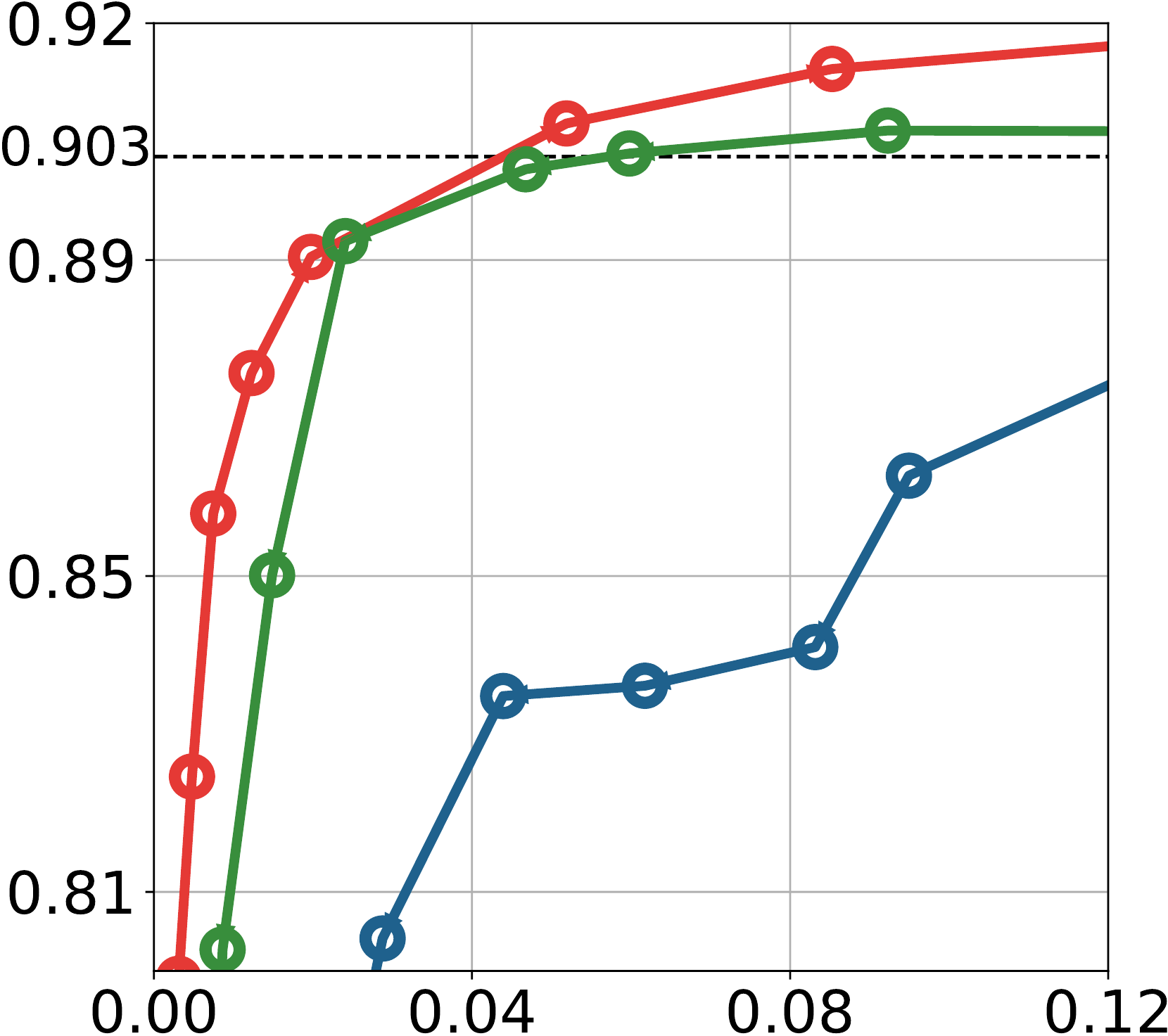} & 
\includegraphics[height =0.22\textwidth]{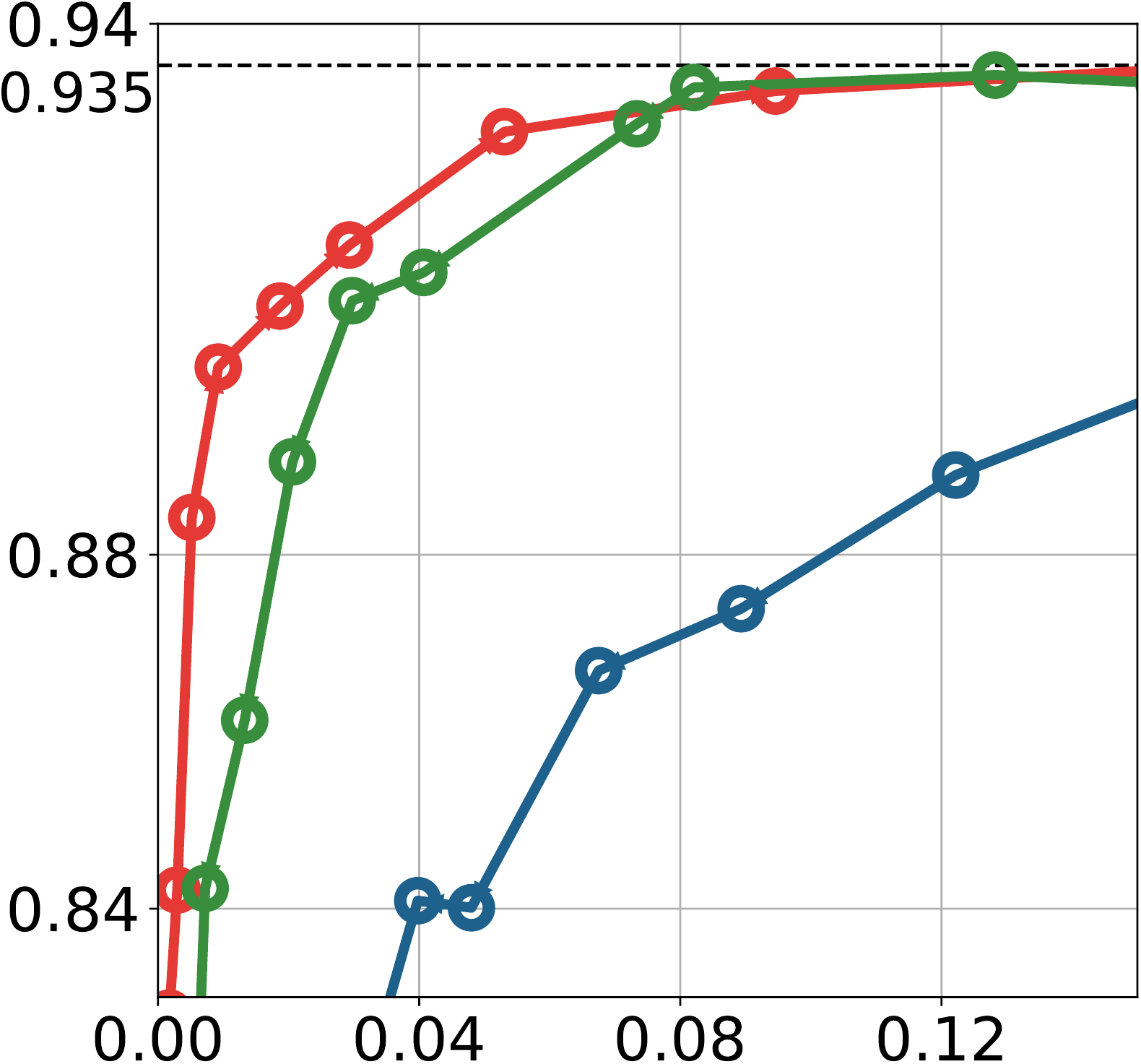} &
\includegraphics[height =0.22\textwidth]{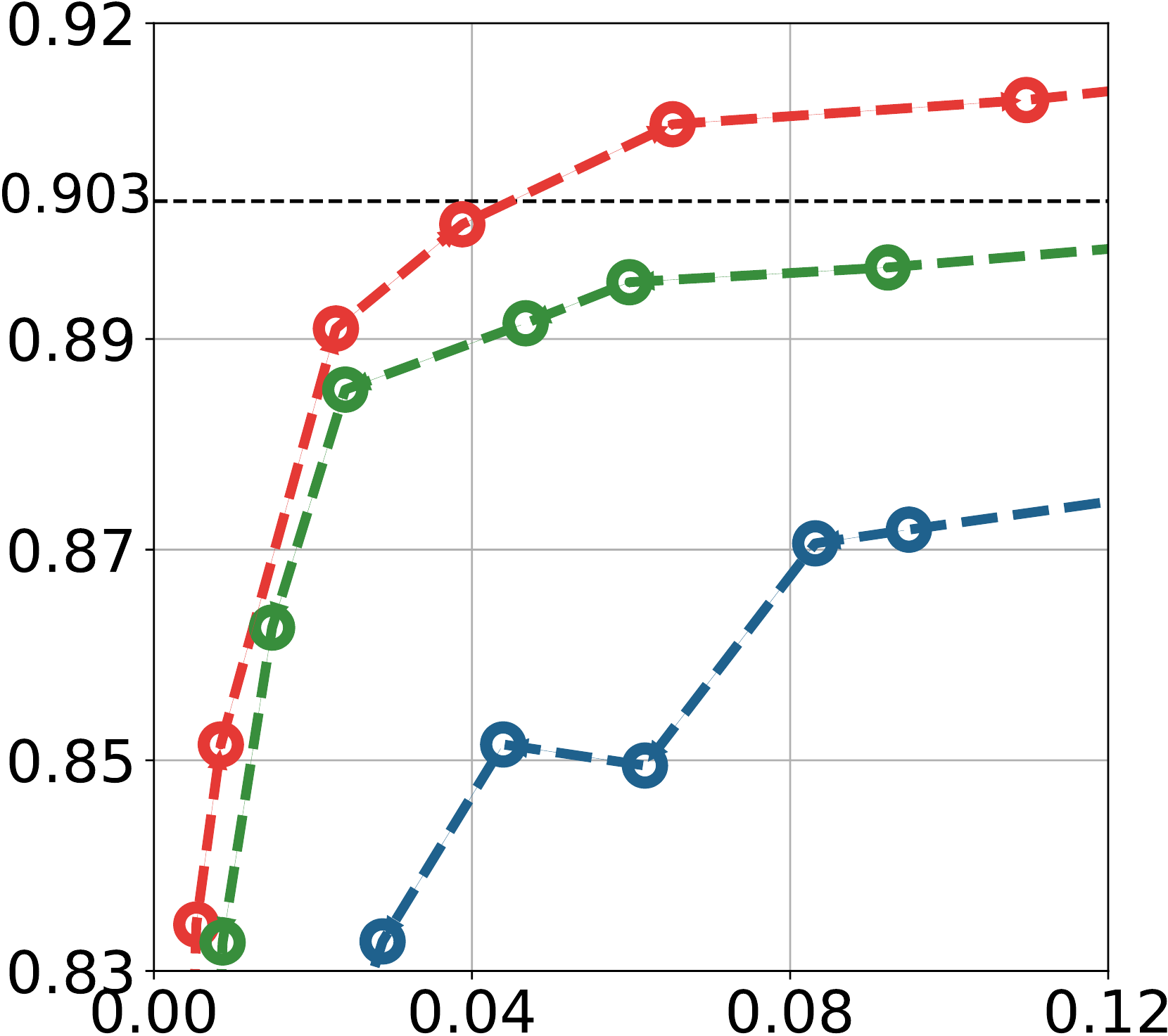} & 
\includegraphics[height =0.22\textwidth]{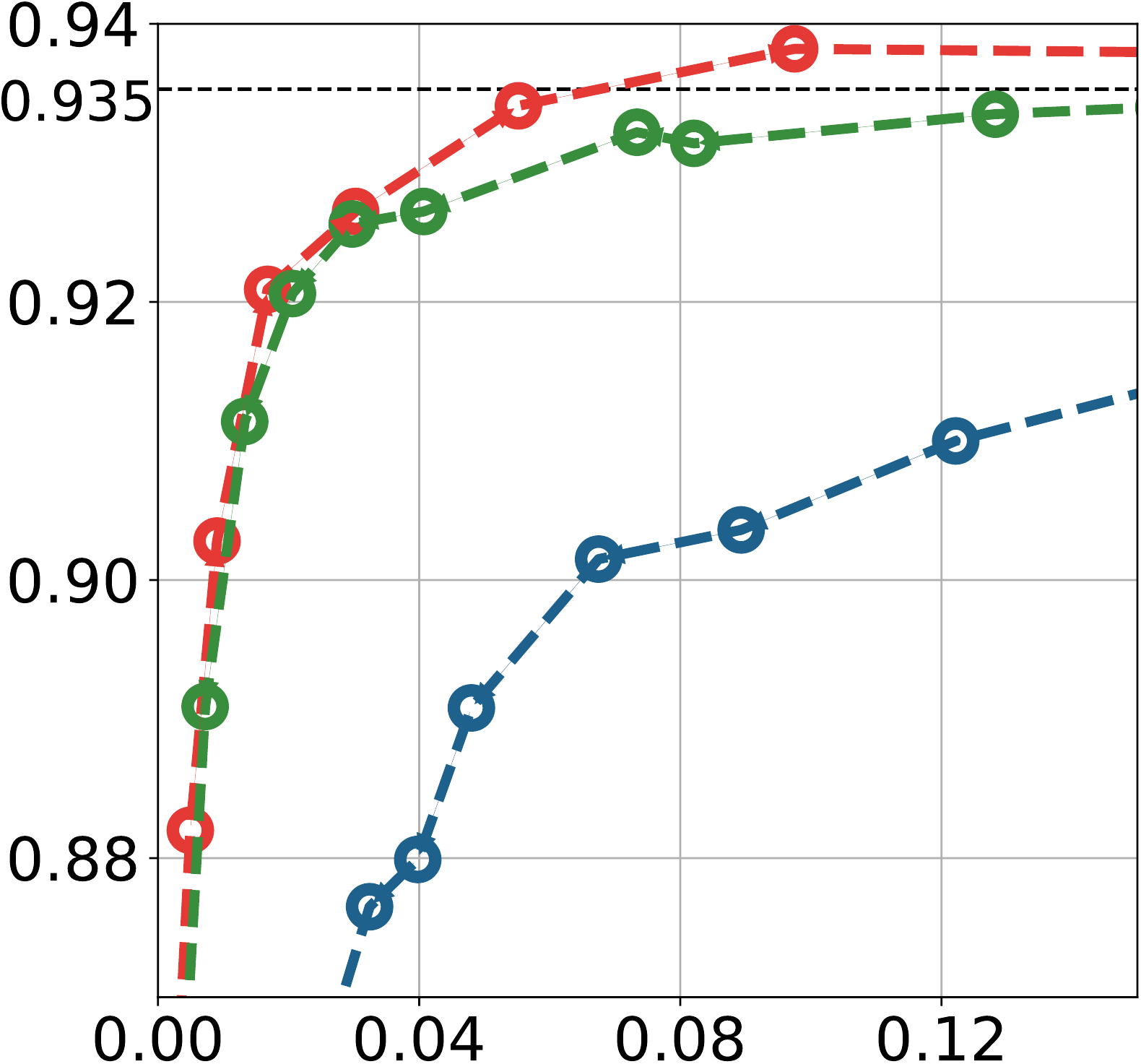}  
\hspace{-5em}
\raisebox{0.6em}{\includegraphics[width =0.12\textwidth]{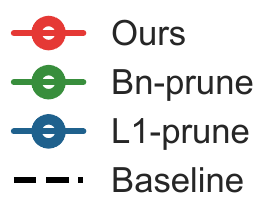}} \\
\small (a){\scriptsize ~~~~~~Ratio} 
&\small (b){\scriptsize  ~~~~~~Ratio} 
&\small (c){\scriptsize  ~~~~~~Ratio} 
&\small (d){\scriptsize  ~~~~Ratio} \\
\end{tabular}
}
\caption{\small Results on CIFAR-10. 
(a)-(b) Results of Algorithm~\ref{alg:main} 
and pruning methods (which successively finetune the neurons after pruning). 
(c)-(d) Results of Algorithm~\ref{alg:main} and prunning methods with retrainning, 
in which we retrain all the weights starting from random initialization after each splitting or pruning step.
The x-axis represents the ratio between the number parameters of the learned models and a full size baseline network.
} 
\label{fig:cifar10}
\end{figure*}
\vspace{-0.5em}

\paragraph{Resource-Efficient Keyword Spotting on Edge Devices}  
Keyword spotting systems aim to detect a particular keyword from a continuous stream of audio. It is  typically deployed on energy-constrained edge devices  
and requires real-time response and high accuracy for good user experience. 
This casts a key challenge of constructing efficient and lightweight neural architectures.  We apply our method to solve this problem, by splitting  
a small model (a compact version of DS-CNN) obtained from \citet{zhang2017hello}. 
See Appendix~\ref{sec:app_kws} for detailed settings. 

Table~\ref{tab:edge} shows the results on 
the Google speech commands benchmark dataset \citep{warden2018speech}, 
in which our method achieves 
significantly higher accuracy than 
the best model (DS-CNN) found by \citet{zhang2017hello}, 
while having ~31\% less parameters and Flops. 
Figure~\ref{fig:kws} shows further comparison 
with Bn-prune \citep{liu2017learning}, which is again inferior to our method.  

\begin{table}[ht]
\begin{minipage}[t]{0.4\linewidth}
\centering
\setlength{\tabcolsep}{2pt}
\scalebox{.9}{
\begin{tabular}{l|ccc}
        \hline
        Method &  Acc & Params (K) & Ops (M)\\
        \hline
        DNN  &  86.94  & 495.7 &  1.0 \\
        CNN &  92.64  & 476.7  & 25.3 \\
        BasicLSTM  &  93.62 &  492.6 & 47.9\\
        LSTM &  94.11  & 495.8& 48.4 \\
        GRU &  94.72  & 498.0 & 48.4 \\
        CRNN &  94.21  & 485.0 & 19.3 \\
        \hline
        DS-CNN & 94.85 & 413.7 &  56.9 \\
         Ours & \textbf{95.36} & \textbf{282.6} & \textbf{39.2} \\ 
        \hline
    \end{tabular}
    }
    \caption{\small Results on keyword spotting. All results are averaged over 5 rounds.}
    \label{tab:edge}
\end{minipage}\hfill
\begin{minipage}[t]{0.56\linewidth}
\centering
\setlength{\tabcolsep}{1pt}
\begin{tabular}{cc}
\hspace{-2.05em}
\raisebox{1.2em}{\rotatebox{90}{\scriptsize Test Accuracy}}
 \includegraphics[height =0.35\textwidth]{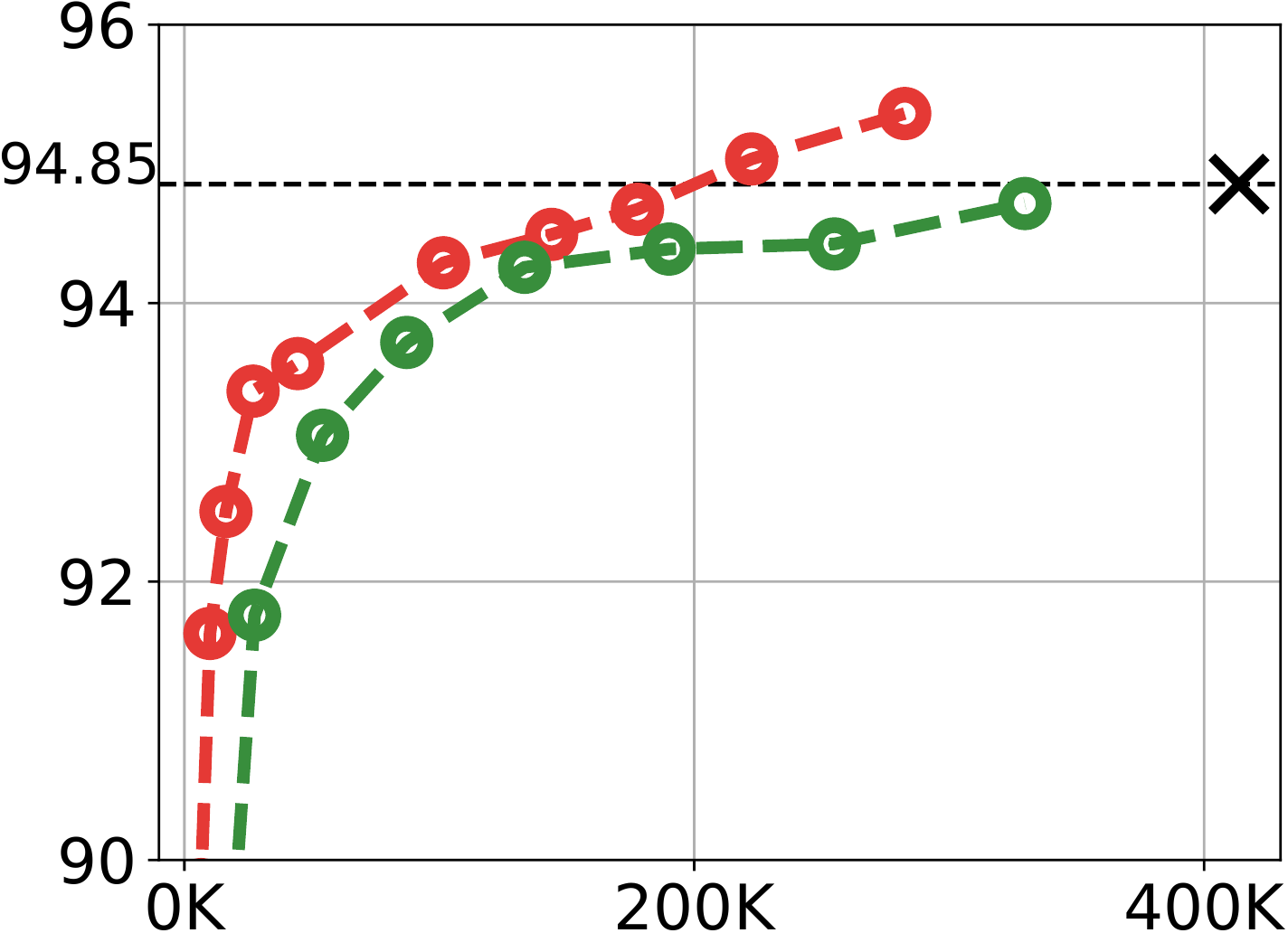} &
    \includegraphics[height =0.35\textwidth]{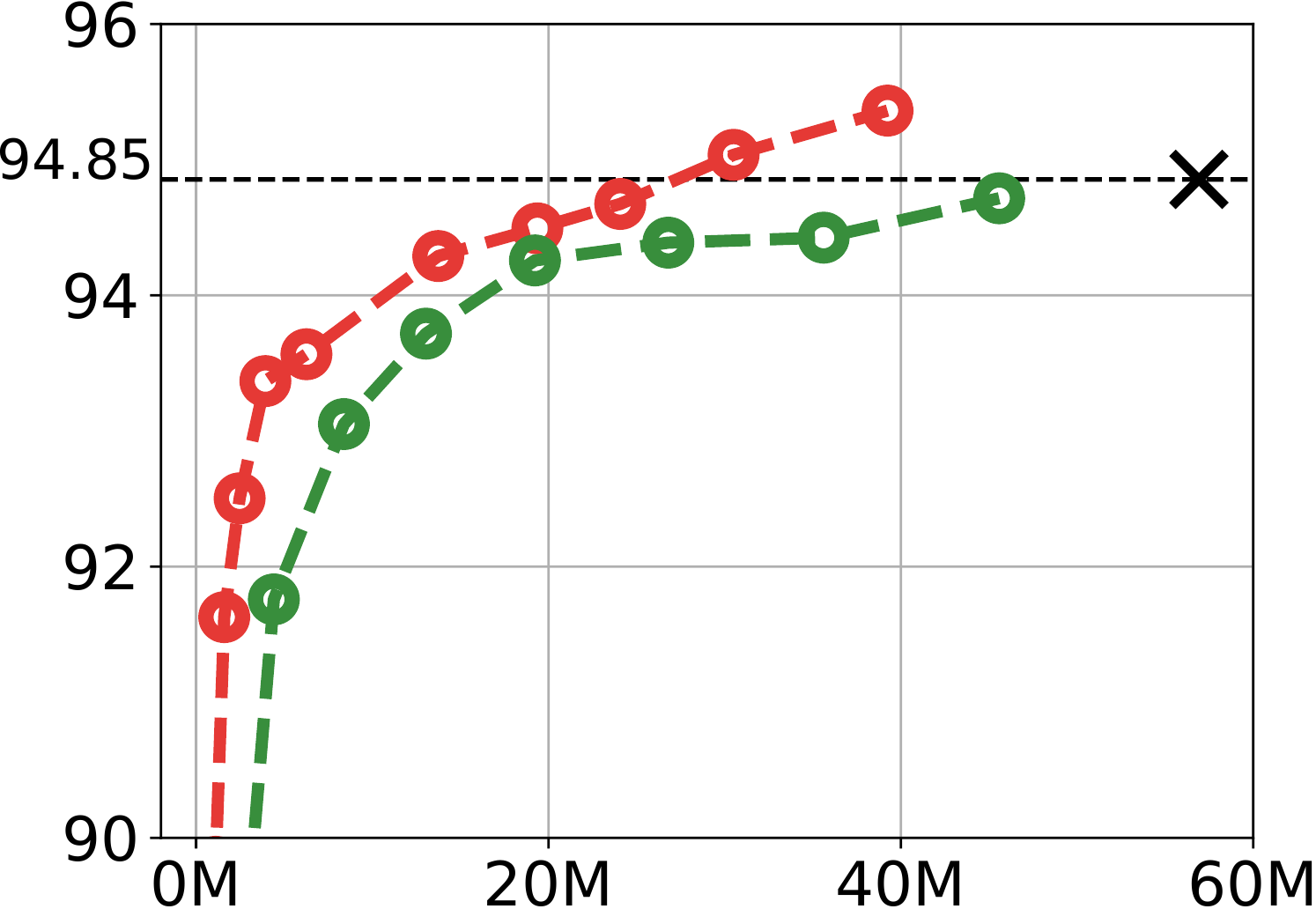} 
    \hspace{-6em}
    \raisebox{1em}{\includegraphics[height =0.10\textwidth]{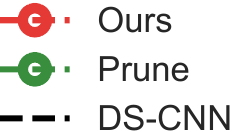}}\\
    ~~~\scriptsize \#Params & ~~~~~~~~~~~~~~\scriptsize \#Ops \\
    \end{tabular}
     \captionof{figure}{\small
     Comparison of accuracy vs. model size (\#Params)  and number of flops (\#Ops) on keyword spotting. 
    }
    \label{fig:kws}
\end{minipage}
\end{table}
\vspace{-1.8em}

\section{Conclusion}
We present a simple approach for progressively training neural networks 
via neuron splitting. 
Our approach highlights a novel view of neural structure optimization as 
continuous functional optimization, and yields a practical procedure with broad applications. For future work, we will further investigate fast gradient descent based approximation of large scale eigen-computation and more theoretical analysis, extensions and applications of our  approach.   

\section*{Acknowledgement}
This work is supported in part by NSF CRII 1830161 and NSF CAREER
1846421. We would like to acknowledge Google Cloud and Amazon Web Services (AWS) for their support.

\bibliography{zzzz}
\bibliographystyle{nips2018}

\newpage \clearpage
\appendix 
\appendix
\section{Proofs}

\subsection{Proofs of Splitting Taylor Expansion}

\begin{proof}[\textbf{Proof of Theorem~\ref{thm:Lthetasp}}]
Taking the gradient of $L(\theta)$ in \eqref{equ:Ltheta} gives $$
\nabla_\para L(\para) = 
 \E[\Phi'(\sigma(\para, x)) \nabla_\para \sigma(\para,x)], 
$$
$$
\nabla^2_{\para\para} L(\para) = 
 \E[\Phi'(\sigma(\para, x)) \nabla_{\para\para}^2 \sigma(\para,x) 
 +  \Phi''(\sigma(\para, x)) \nabla_\para \sigma(\para,x)^{\otimes 2}   ],  
$$
where $\Phi'(\cdot)$ is the derivative of $\Phi(\cdot)$ (which is a univariate function), and 
$\nabla_\para  \sigma(\para, x)^{\otimes 2} :=
\nabla_\para  \sigma(\para, x)\nabla_\para  \sigma(\para, x)^\top.  
$ 

When $\theta$ is split into $ 
\{w_i, \theta_i\}_{i=1}^m$, the augmented loss function is 
$$
\Lm(\vv \para, \vv \wts) =  \E \left [\Phi \left (  \sum_{i=1}^m \wts_i \sigma(\para_i, x)  \right ) \right],  
$$
where $\vv w = [w_1,\ldots, w_m]$ and $\vv \theta = [\theta_1, \ldots, \theta_m].$ The weights should satisfy $\sum_{i=1}^m w_i = 1$ and $w_i \geq 0$. 
In this way, we have $\Lm(\vv \para, \vv \wts) = L(\theta)$ when
 $\vv\theta = [\theta, \ldots, \theta] = \theta\one_m$. 

Taking the gradient of $\Lm(\vv\theta, \vv w)$ w.r.t. $\theta_i$ when $\vv\theta = \theta \one_m$, we have 
$$
\nabla_{\para_i}\Lm\left(\theta\one_m, ~ \vv\wts \right) =
 \E\left [\Phi' \left ( \sigma(\para, x)  \right )  
 \wts_i \nabla_\para \sigma(\para, x) \right] = \wts_i \nabla_\para L(\para).
$$
Taking the second derivative, we get 
\begin{align*} 
\nabla_{\para_i,\para_i} \Lm(\theta\one_m, ~ \vv\wts) 
& =  \E\left [ 
 \Phi' \left (  \sigma(\para, x)  \right )   \wts_i \nabla_{\para,\para}^2 \sigma(\para, x) 
 ~+~ 
 \Phi'' \left (  \sigma(\para, x)   \right )  
 \wts_i^2 \nabla_\para \sigma(\para, x)^{\otimes 2} 
  \right] \\
&  := \wts_i A(\theta) + \wts_i^2 B(\theta), 
\end{align*}
where 
\begin{align*} 
A(\theta) := \E\left [ 
 \Phi' \left (  \sigma(\para, x)  \right )    \nabla_{\para,\para}^2 \sigma(\para, x)  \right ], 
 & &&&
B(\theta) := \E\left [ \Phi'' \left (  \sigma(\para, x)   \right )  \nabla_\para \sigma(\para, x)^{\otimes 2} 
  \right]. 
\end{align*}
Note that we have $\nabla^2_{\theta\theta} L(\theta) = A(\theta) + B(\theta)$ following this definition.  

For $i \neq j$, we have 
$$
\nabla_{\para_i \para_j}\Lm(\para\one_m, ~ \vv\wts) = 
 \E\left [\Phi'' \left ( \sigma(\para, x)   \right ) 
  \wts_i \wts_j \nabla_\para \sigma(\para, x)^{\otimes 2} \right] = \wts_i \wts_j B(\theta). 
$$

For $\vv\theta = [\theta_1, \ldots, \theta_m]$, assume  
 $\theta_i=\theta + \epsilon \delta_i$, and define $\bar \delta = \sum_{i=1}^m w_i \delta_i$ to be the average displacement.
Therefore, $\vv\theta = \theta \vv 1_m + \epsilon \vv \delta $. 
Using the Taylor expansion of $\Lm(\theta \vv 1_m + \epsilon \vv \delta, \vv w)$ w.r.t. $\epsilon$ at $\epsilon = 0$, we have  
\begin{align*} 
\Lm(\vv\theta, ~\vv w) - L(\para) 
& = \Lm(\theta\vv 1_m + \epsilon \vv \delta, ~\vv w) - L(\para) \\
& =   \epsilon \sum_{i=1}^m \nabla_{\theta_i} \Lm(\theta\vv 1_m, \vv w)^\top \delta_i  
+   \frac{\epsilon^2}{2} \sum_{ij=1}^m 
\delta_i^\top (\nabla_{\theta_i,\theta_j}^2 \Lm(\theta\vv 1_m, \vv w)) \delta_i  
+ \obig(\epsilon^3)   \\
& = \epsilon\sum_{i=1}^m w_i \nabla L(\para)^\top \delta_i  
+ \frac{\epsilon^2}{2}\sum_{i=1}^m w_i \delta_i^\top A(\theta)  \delta_i 
+ \frac{\epsilon^2}{2}\sum_{ij=1}^m w_i w_j \delta_i^\top B(\theta) \delta_j + \obig(\epsilon^3) \\
& =   \epsilon \nabla L(\para)^\top \bar \delta 
+  \frac{\epsilon^2}{2}  \sum_{i=1}^m  w_i \delta_i^\top A(\theta) \delta_i
+ \frac{\epsilon^2}{2}  \bar \delta^\top B(\theta) \bar \delta 
+ \obig(\epsilon^3) \\
& = 
  \epsilon \nabla L(\theta)^\top \bar \delta 
  + \frac{\epsilon^2}{2} \bar \delta^\top (A(\theta)+B(
  \theta)) \bar \delta  
  + \frac{\epsilon^2}{2}  \sum_{i=1}^m 
  ( \delta_i^\top A \delta_i - \bar \delta^\top A \bar \delta)     + \obig(\epsilon^3) \\
  &=\epsilon \nabla L(\theta)^\top \bar \delta 
  + \frac{\epsilon^2}{2} \bar \delta^\top \nabla^2 L(\theta)  \bar \delta  
  + \frac{\epsilon^2}{2}  \sum_{i=1}^m  
  ( \delta_i - \bar\delta)^\top A(\theta) (\delta_i - \bar \delta)     + \obig(\epsilon^3). 
\end{align*}
This completes the proof. 
\end{proof}

\begin{proof}[\textbf{Proof of Theorem~\ref{thm:opt}}]
Recall that 
$$
 \II(\vv\delta, \vv w; \theta) = 
 \frac{\epsilon^2}{2}
 \sum_{i=1}^m w_i \delta_i ^\top S(\theta) \delta_i, 
$$
with $\sum_i w_i =1$, $w_i \geq 0$ and $\norm{\delta_i}= 1$. 
Since $\delta_i ^\top S(\theta) \delta_i \geq \lambda_{min}(S(\theta)) \norm{\delta_i}^2 = \lambda_{min}(S(\theta)) $, 
it is obvious that 
$$
\II(\vv\delta, \vv w; \theta) = 
 \frac{\epsilon^2}{2}
\sum_{i=1}^m w_i \delta_i ^\top S(\theta) \delta_i 
\geq 
 \frac{\epsilon^2}{2}
\sum_{i=1}^m w_i \lambda_{min}(S(\theta)) = 
 \frac{\epsilon^2}{2} \lambda_{min}(S(\theta)). 
$$
On the other hand, this lower bound is achieved by setting $m=2$, $w_1=w_2 = 1/2$ and $\delta_1 = -\delta_2 = v_{min}(S(\theta)).$ This completes the proof. 
\end{proof}

\begin{proof}[\textbf{Proof of Theorem~\ref{thm:general}}]
Step 1: We first consider the case with no average displacement, that is, $\mu^\supell = 0$. 
 In this case, Lemma~\ref{lem:decomp} below gives 
  \begin{align} 
     \label{equ:decomptmp}
 \L(\vv\theta^{[1:n]}, \vv w^{[1:n]})  = 
   L(\theta^{[1:n]})  +  \sum_{\ell=1}^n   \left (\L(\vv\theta^{[1:n]}\onlyell, \vv w^{[1:n]})  - L(\theta^{[1:n]}) \right) + \obig(\epsilon^3), 
  \end{align}
  where $\vv\theta^{[1:n]}\onlyell$ denotes the augmented parameters obtained when we only split the $\ell$-th neuron, while keeping all the neurons unchanged. 
  Applying Theorem~\ref{thm:Lthetasp}, we have for each $\ell$, 
  $$
  \L(\vv\theta^{[1:n]}\onlyell, \vv w^{[1:n]})  - L(\theta^{[1:n]}) = 
  \frac{\epsilon^2}{2}  \II_\ell(\vv\delta^\supell,  \vv w^\supell; ~\theta^{[1:n]})~+~\obig(\epsilon^3). 
  $$
  Combining this with \eqref{equ:decomptmp} yields the result. 
 
Step 2: We now consider the more general case when $\mu^{[1:n]}\neq 0$. Let $\tilde \theta^{[1:n]} = \theta^{[1:n]} + \epsilon \mu^{[1:n]}$. Applying the result above 
on $\tilde \theta^{[1:n]}$, we have 
\begin{align*} 
\L\left (\vv\theta^{[1:n]}, \vv w^{[1:n]} \right ) 
& = L\left (\tilde \theta^{[1:n]}\right ) + \frac{\epsilon^2}{2} D\left (\tilde \theta^{[1:n]}\right) +\obig(\epsilon^3)\\
\end{align*}
where $D\left (\tilde \theta^{[1:n]}\right ) : = \sum_{\ell=1}^n\II_\ell(\vv\delta^\supell,  \vv w^\supell; ~\tilde \theta^{[1:n]}).$ 
Therefore,
\begin{align*} 
\L\left (\vv\theta^{[1:n]}, \vv w^{[1:n]} \right ) 
& = L\left (\tilde \theta^{[1:n]}\right ) + \frac{\epsilon^2}{2} D(\tilde \theta^{[1:n]}) +\obig(\epsilon^3)\\
& = L\left (\tilde \theta^{[1:n]}\right ) + 
\frac{\epsilon^2}{2} D(\theta^{[1:n]})
+\frac{\epsilon^2}{2} (D(\tilde \theta^{[1:n]})-D(\theta^{[1:n]}) )+\obig(\epsilon^3)\\
& = 
L\left (\tilde \theta^{[1:n]}\right ) + \frac{\epsilon^2}{2} D(\theta^{[1:n]}) + \obig(\epsilon^3) \ant{\text{because $\theta^{[1:n]} - \tilde \theta^{[1:n]} = \obig(\epsilon)$}}
\\
& = 
L\left (\theta^{[1:n]} + \epsilon \mu^{[1:n]}\right ) + \frac{\epsilon^2}{2} D(\theta^{[1:n]}) + \obig(\epsilon^3),\\
\end{align*}
where 
$D\left ( \theta^{[1:n]}\right ) : = \sum_{\ell=1}^n\II_\ell(\vv\delta^\supell,  \vv w^\supell; ~ \theta^{[1:n]}).$ 
This completes the proof. 
\end{proof}

\begin{lem}\label{lem:decomp}
Let $\theta^{[1:n]}$ be the parameters of $n$ neurons. 
Recall that we assume $\theta^{[\ell]}$ is split into $m_\ell$ off-springs with 
parameters $\vv\theta^{[\ell]} = \{\theta_i^{[\ell]}\}_{i=1}^{m_\ell}$ and weights $\vv w^{[\ell]} = \{w_{i}^\supell\}_{i=1}^{m_\ell}$,  which satisfies $\sum_{i=1}^{m_\ell} w_i^\supell = 1$.  
Let $\theta_i^\supell  = \theta^\supell + \epsilon \delta_i^\supell$, where $\delta_i^\supell$ is the perturbation on the $i$-th off-spring of the $\ell$-th neuron.   
Assume $\bar \delta^{[\ell]} := \sum_{i=1}^{m_\ell} w_i^{[\ell]} \delta_i^\supell = 0$, that is, 
 the average displacement of all the neurons is zero.  

Denote by $\vv\theta^{[1:n]}_\ell$ the augmented parameters we obtained by only splitting the $\ell$-th neuron while keeping all the other neurons unchanged, that is,  we have 
$\theta_{\ell,i}^{[\ell]} = \theta^{[\ell]} + \epsilon \delta_i^\supell$ for $i=1,\ldots, m_\ell$, and 
$\theta_{\ell,i}^{[\ell']} = \theta^{[\ell']}$ for all $\ell' \neq \ell$ and $i = 1,\ldots, m_{\ell'}$. 
Assume the third order derivatives of $\L(\vv\theta^{[1:n]}, \vv w^{[1:n]})$ are bounded. We have 
 \begin{align*}
 \L(\vv\theta^{[1:n]}, \vv w^{[1:n]})  = 
   L(\theta^{[1:n]})  +  \sum_{\ell=1}^n   \left (\L(\vv\theta^{[1:n]}\onlyell, \vv w^{[1:n]})  - L(\theta^{[1:n]}) \right) + \obig(\epsilon^3). 
  \end{align*}
\end{lem}

\begin{proof}
Define 
 $$
F := \Big(\L(\vv\theta^{[1:n]}, \vv w^{[1:n]}) 
 -  L(\theta^{[1:n]})   \Big)
 -  \sum_{\ell=1}^n   \left (\L(\vv\theta^{[1:n]}\onlyell, \vv w^{[1:n]})  - L(\theta^{[1:n]}) \right).  
 $$
By Taylor expansion, 
$$
 F = 
 \epsilon \nabla_\epsilon F \big|_{\epsilon=0} +\frac{\epsilon^2}{2}  
 \nabla_{\epsilon\epsilon} F \big|_{\epsilon=0} + \obig (\epsilon^3). 
$$

It is obvious to see that the first order derivation  $\nabla_\epsilon F \big|_{\epsilon=0}$ equals zero because of the correction terms.  Specifically, 
\begin{align*}
\nabla_\epsilon F\big|_{\epsilon=0} 
= \sum_{\ell=1}^n \sum_{i=1}^{m_\ell}
 \nabla_{\theta_i^\supell } \L(\vv\theta^{[1:n]}, \vv w^{[1:n]})^\top \delta_i^\supell 
 \bigg|_{\epsilon=0} 
 ~~ 
- \sum_{\ell=1}^n \sum_{i=1}^{m_\ell} 
\nabla_{\theta_i^\supell} 
\L(\vv\theta^{[1:n]}, \vv w^{[1:n]})^\top 
\delta_i^\supell 
\bigg|_{\epsilon=0}  = 0. 
\end{align*}
For the second order derivation, define 
$$
A_{\ell,\ell'} = \nabla_{\theta^{[\ell]} \theta^{[\ell']}} L(\theta^{[1:n]}). 
$$
For any $\ell \neq \ell'$, 
we have from \eqref{equ:lell}
and \eqref{equ:LLLLL} that 
$$
\nabla_{\theta^\supell_i \theta^{[\ell']}_{i'}} \L(\vv\theta^{[1:n]}, \vv w^{[1:n]})\bigg|_{\epsilon=0}    
= w_{i}^{[\ell]}  w_{i'}^{[\ell']} \nabla_{\theta^{[\ell]} \theta^{[\ell']}} L(\theta^{[1:n]}) 
=  w_{i}^{[\ell]}  w_{i'}^{[\ell']} A_{\ell,\ell'}. 
$$
Therefore, we have  
\begin{align*}
\nabla_{\epsilon\epsilon} F\big|_{\epsilon=0}
& =  \sum_{\ell\neq \ell'}\sum_{i=1}^{m_\ell} \sum_{i'=1}^{m_{\ell'}}
(\delta_i^\supell )^\top 
\nabla_{\theta^\supell_i \theta^{[\ell']}_{i'}} \L(\vv\theta^{[1:n]}, \vv w^{[1:n]})\bigg|_{\epsilon=0}    
\delta_{i'}^{[\ell']}  \\
& =  
\sum_{\ell\neq \ell'}\sum_{i=1}^{m_\ell} \sum_{i'=1}^{m_{\ell'}}
w_{i}^{[\ell]}  w_{i'}^{[\ell']}  (\delta_i^\supell )^\top A_{\ell,\ell'} \delta_{i'}^{[\ell']}  \\
& = \sum_{\ell \neq \ell'} (\bar \delta^{[\ell]} )^\top A_{\ell,\ell'} \bar\delta^{[\ell']} \\
& = 0 \ant{\text{because $\bar \delta^\supell = 0$}},
\end{align*}
where $\nabla_{\epsilon\epsilon} F\big|_{\epsilon=0}$ 
only involves cross derivatives $\nabla_{\theta^\supell_i \theta^{[\ell']}_{i'}} \L(\vv\theta^{[1:n]}, \vv w^{[1:n]})$ with $\ell \neq \ell'$,
because all the terms with $\ell = \ell'$ are cancelled due to the correction terms. 
\end{proof}

\subsection{Proofs of $\infty$-Wasserstein Steepest Descent} 
\label{sec:wassappendix}
Recall that $p$-Wasserstein distance is 
$$
W_{p}(\rho, \rho')  = \inf_{\gamma\in \Pi(\rho,\rho')} \E_{(\theta,\theta')\sim \gamma}[\norm{\theta-\theta'}^p]^{1/p}.
$$
When $p\to +\infty$, we obtain $\infty$-Wasserstein distance, 
\begin{align}\label{equ:winftyapp}
W_\infty(\rho,\rho') = 
\inf_{\gamma \in \Pi(\rho,\rho')}
\esssup_{(\theta,\theta')\sim \gamma} 
\norm{\theta - \theta'}, 
\end{align}
where $\esssup$ denotes
essential supremum; 
it is the minimum value $c$  with $\gamma(\norm{\theta-\theta'} \geq c) = 0$.  

In the proof, we denote by $\gamma_{
\rho,\rho'}$ an optimal solution of  $\gamma$ in \eqref{equ:winftyapp}, that is, 
$$
\gamma_{\rho,\rho'}\in \arg\inf_{\gamma \in \Pi(\rho,\rho')}
\esssup_{(\theta,\theta')\sim \gamma} 
\norm{\theta - \theta'}. 
$$
$\gamma_{\rho,\rho'}$ is called an
$\infty$-Wasserstein optimal coupling of $\rho$ and $\rho'$.  
Denote by $\mu_{\rho,\rho'}(\theta)$ and $\Sigma_{\rho,\rho'}(\theta)$ the mean and covariance matrix of $(\theta'-\theta)$ under $\gamma_{\rho,\rho'}$, conditional on $\theta$, that is, 
\begin{align*}
\mu_{\rho,\rho'}(\theta) = \E_{\gamma_{\rho,\rho'}}[(\theta'-\theta)~|~\theta] 
&&
\Sigma_{\rho,\rho'}(\theta) = \cov_{\gamma_{\rho,\rho'}}\left [ (
\theta'-\theta)~|~\theta\right].
\end{align*}
It is natural to expect that we can upper bound the magnitude of both $\mu_{\rho,\rho'}(\theta)$ and $\Sigma_{\rho,\rho'}(\theta)$ by
the $\infty$-Wasserstein distance. 
\begin{lem}
Following the definition above, we have 
\begin{align*}
 \norm{\mu_{\rho,\rho'}(\theta)} 
\leq  W_{\infty}(\rho,\rho'), &&
\lambda_{max}(\Sigma_{\rho,\rho'}(\theta))  \leq  W_{\infty}(\rho,\rho')^2, 
\end{align*}
almost surely for $\theta\sim \rho$. 
\end{lem}
\begin{proof} 
We have
$$
\norm{\mu_{\rho,\rho'}(\theta)} 
\leq \esssup_{\gamma_{\rho,\rho'}} \norm{\theta-\theta'} 
= W_{\infty}(\rho,\rho'), 
$$
almost surely for $\theta \sim \rho$. And 
\begin{align*} 
\lambda_{max}(\Sigma_{\rho,\rho'}(\theta))
& =  \max_{\norm{v}=1}\var_{\gamma_{\rho,\rho'}}\left [ v^\top (\theta' - \theta) ~|~\theta  \right ] \\
& \leq  \max_{\norm{v}=1}\E_{\gamma_{\rho,\rho'}}\left [ \left (v^\top (\theta' - \theta)\right)^2 ~|~\theta  \right ] \\
& \leq \esssup_{\gamma_{\rho,\rho'}} \norm{\theta-\theta'}^2  \\
& = W_{\infty}(\rho,\rho')^2. 
\end{align*}
\end{proof} 


\begin{thm} \label{thm:wasseroneorder}
Define $G_\rho(\theta) 
= \E_{x\sim\mathcal D}\left [\nabla \Phi\left(  \E_{\rho}[\sigma( \theta,x)]  \right) \nabla\sigma( \theta,x) \right]$. 
For two distributions $\rho$ and $\rho'$ and their $\infty$-Wasserstein optimal coupling $\gamma_{\rho,\rho'}$. 
We have 
\begin{align} \label{equ:Lrhodecomp}
\begin{split} 
\L[\rho'] 
& = \L[\rho] 
~+~
\E_{\theta\sim \rho}\left [
G_\rho(\theta) ^\top  \mu_{\rho,\rho'}(\theta) \right ]
~+~ \obig((\D_\infty(\rho,\rho')^2). 
\end{split}
\end{align}
\end{thm}
\begin{proof}
We write $\gamma = \gamma_{\rho,\rho'}$ for convenience. 
Denote by $
\nabla \sigma(\theta,x) = \nabla_{\theta} \sigma(\theta,x)$
and $\nabla^2 \sigma(\theta,x) = \nabla_{\theta\theta} ^2\sigma(\theta,x)$ the first and second order derivatives of $\sigma$ in terms of its first variable. 

For $(\theta,\theta')\sim \gamma$, introduce $\theta_\eta = \eta\theta' + (1-\eta)\theta$, whose distribution is  denoted by $\rho_\eta$. We have $\rho_0 = \rho$ and $\rho_1 = \rho'$. Taking Taylor expansion of $\L[\rho_\eta]$ w.r.t. $\eta$, we have 
$$
\L[\rho'] = \L[\rho] + \nabla_\eta \L[\rho_\eta] \bigg |_{\eta = 0} + \frac{1}{2}\nabla_{\eta\eta}^2 \L[\rho_\eta] \bigg |_{\eta = \xi},   
$$
where $\xi$ is a number between $0$ and $1$. We just need to calculate the derivatives.  
For the first order derivative, we have 
\begin{align*}
\nabla_\eta \L[\rho_\eta] \bigg |_{\eta=0} 
& =  \nabla_\eta 
 \E_{x\sim\mathcal D}\left [ \Phi\left(  \E_{\gamma}[\sigma(\eta\theta'+(1-\eta)  \theta,x)]  \right) \right ]\bigg |_{\eta=0}  \\
 & =  
 \E_{x\sim\mathcal D}\left [  \Phi'\left(  \E_{\gamma}[\sigma(\theta_\eta,x)]  \right) 
 \E_{\gamma}[\nabla\sigma(\theta_\eta,x) ^\top (\theta' - \theta) ] 
 \right ]\bigg |_{\eta=0}  \\
 & =  
 \E_{x\sim\mathcal D}\left [  \Phi'\left(  \E_{\gamma}[\sigma( \theta,x)]  \right) 
  \E_{\gamma}[ \nabla \sigma( \theta,x) ^\top (\theta' - \theta) ] 
 \right ] \\
 & = \E_{\gamma}[G_\rho(\theta)^\top(\theta' - \theta)] \\
 & = \E_{\rho}\left [
G_\rho(\theta) ^\top  \mu_{\rho,\rho'}(\theta) \right ],
\end{align*}
where we used the derivation of $G_\rho(\theta)$. 

For the second order derivative, we have 
\begin{align*} 
 \nabla_{\eta\eta}^2 \L[\rho_\eta] \bigg |_{\eta = \xi}    
& = 
\nabla_\eta(\nabla_{\eta} \L[\rho_\eta] )\bigg |_{\eta = \xi}   \\ 
& = \nabla_{\eta} \E_{x\sim\mathcal D}\left [  \Phi'\left(  \E_{\gamma}[\sigma(\theta_\eta,x)]  \right) 
 \E_{\gamma}[\nabla\sigma(\theta_\eta,x) ^\top (\theta' - \theta) ] 
 \right ] \bigg |_{\eta=\xi}   \\
& =  \E_{x\sim\mathcal D} \left [  \Phi''\left(  \E_{\gamma}[\sigma(\theta_\eta,x)]  \right) 
 (\E_{\gamma}[\nabla\sigma(\theta_\eta,x)  (\theta' - \theta) ])^2 
 \right ]  \\
 & ~~~~~~~~~~
 +~~\E_{x\sim\mathcal D}\left [\Phi'\left( \E_{\gamma}[\sigma(\theta_\eta,x)]  \right) 
 \E_{\gamma}[(\theta'-\theta)^\top\nabla^2 \sigma(\theta_\eta,x)  (\theta' - \theta) ] 
 \right ]  
 \bigg |_{\eta=\xi}  \\
 & = 
  \E_{\gamma}[(\theta'-\theta)^\top  T_\rho(\theta_{\xi}) (\theta' - \theta) ]  ~+~ 
 \E_{\gamma}[(\theta'-\theta)^\top S_\rho(\theta_\xi) (\theta' - \theta) ] \\
 & = 
  \E_{\gamma}[(\theta'-\theta)^\top \left ( T_\rho(\theta_\xi ) + S_\rho(\theta_\xi) \right )  (\theta' - \theta) ]
\end{align*}
where we define  $T_\rho(\theta_\xi) := 
\E_{x\sim\mathcal D} 
\left [  \Phi''\left(  \E_{\gamma}[\sigma(\theta_\xi,x)]   \right )
\nabla\sigma(\theta_\xi,x)^{\otimes 2}  \right ] $. 
Denote by $\lambda_* := \sup_{\xi \in [0,1]}\lambda_{max}(T_\rho(\theta_\xi) + S_\rho(\theta_\xi)).$
We have 
\begin{align*} 
 \nabla_{\eta\eta}^2 \L[\rho_\eta] \bigg |_{\eta = \xi} 
 & \leq  \lambda_* \E_\gamma \left [\norm{\theta' -\theta}^2\right ] \\
 & = \obig\left (\E_\gamma \left [\norm{\theta' -\theta}^2\right ]\right)   \\
 & = \obig\left (\left [\esssup_\gamma \norm{\theta' -\theta}\right ]^2\right)   \\
 & = \obig(\D_\infty(\rho,\rho')^2). 
 \end{align*}
This completes the proof. 
\end{proof}

\begin{thm} \label{thm:wasser}
 For two distributions $\rho$ and $\rho'$,
 denote by $\gamma_{\rho,\rho'}$ their $\infty$-Wasserstein optimal coupling, and $\mu_{\rho,\rho'}(\theta)$ 
and $\Sigma_{\rho,\rho'}(\theta)$ the mean and covariance matrix of $(\theta'-\theta)$ under $\gamma_{\rho,\rho'}$, conditional on $\theta$, respectively. 
Denote by $(I+\mu_{\rho,\rho'})\sharp \rho$ the distribution of $\theta +\mu_{\rho,\rho'}(\theta)$ when $\theta\sim \rho$. We have 
\begin{align} \label{equ:Lrhodecomp}
\L[\rho'] = \L\left [(I+\mu_{\rho,\rho'})\sharp \rho \right ]  
~+~
\E_{\theta\sim \rho} \left [ \frac{1}{2} \trace\big ( S_\rho(\theta)^\top 
\Sigma_{\rho,\rho'}(\theta) \big )\right ] 
~+~ \obig((\D_\infty(\rho,\rho'))^3) 
\end{align}
where $S_\rho(\theta) = \E_{x\sim \mathcal D} \left [ \Phi'(f_\rho(x)) \nabla_{\theta\theta}^2\sigma(\theta,x) \right ].$ 
The first and second terms capture the effect of displacement and splitting, respectively. 
 
\end{thm}
 
\begin{proof}[\textbf{Proof of Theorem~\ref{thm:wasser}}]
We use $\gamma := \gamma_{\rho,\rho'}$ for notation convenience. 
Denote by $\tilde \theta = \theta + \mu_{\rho,\rho'}(\theta)$ and $\tilde \rho = (I+\mu_{\rho,\rho'})\sharp \rho$ the distribution of $\tilde \theta$ when $\theta\sim \rho$. 
Recall that for $(\theta,\theta') \sim \gamma$, we have  
\begin{align}\label{equ:zerodfjdif}
\E_\gamma\left [\theta' - \tilde \theta ~\bigg|~\theta \right ] = \E_{\gamma} [\theta'-\theta - \mu_{\rho,\rho'}(\theta)~|~\theta] = 0,
\end{align}
\begin{align}
    \label{equ:zerodfjdif2}
\Sigma_{\rho,\rho'}(\theta)=\E_{\gamma}\left [ (\theta'-\tilde \theta) (\theta' - \tilde \theta )^\top \bigg|~  \theta \right ]. 
\end{align}

Introduce $\theta_\eta = \eta\theta' + (1-\eta)\tilde \theta$. Denote by $\rho_\eta$ the distribution of $\theta_\eta$. This gives $\rho' = \rho_1$ and $\tilde \rho = \rho_0$.
 We have 
$$
\L[\rho'] =  \L[\tilde \rho] 
+ \nabla_\eta \L[\rho_\eta] \bigg |_{\eta=0}  
+ \frac{1}{2} \nabla_{\eta\eta}^2 \L[\rho_\eta] \bigg |_{\eta = 0} 
+   \frac{1}{6} \nabla_{\eta\eta\eta}^3 \L[\rho_\eta] \bigg |_{\eta = \xi},
$$
where $\xi$ is a number between $0$ and $1$. We just need to evaluate these derivatives. 
For the first order derivative, we have 
\begin{align*}
\nabla_\eta \L[\rho_\eta] \bigg |_{\eta=0} 
& =  \nabla_\eta 
 \E_{x\sim\mathcal D}\left [ \Phi\left(  \E_{\gamma}[\sigma(\eta\theta'+(1-\eta)\tilde \theta,x)]  \right) \right ]\bigg |_{\eta=0}  \\
 & =  
 \E_{x\sim\mathcal D}\left [  \Phi'\left(  \E_{\gamma}[\sigma(\theta_\eta,x)]  \right) 
 \E_{\gamma}[\nabla \sigma(\theta_\eta,x) ^\top (\theta' - \tilde\theta) ] 
 \right ]\bigg |_{\eta=0}  \\
 & =  
 \E_{x\sim\mathcal D}\left [  \Phi'\left(  \E_{\gamma}[\sigma(\tilde \theta,x)]  \right) 
  \E_{\gamma}[ \nabla \sigma(\tilde \theta,x) ^\top (\theta' - \tilde\theta) ] 
 \right ] \\
 & = 0, 
\end{align*}
where the last step uses \eqref{equ:zerodfjdif}. 
Here $\nabla \sigma$ denote the derivatives w.r.t. its first variables.  

For the second order derivative, we have 
\begin{align} 
& \nabla_{\eta\eta}^2 \L[\rho_\eta] \bigg |_{\eta = 0}    
 = 
\nabla_\eta(\nabla_{\eta} \L[\rho_\eta] )\bigg |_{\eta = 0}   \notag \\ 
& = \nabla_{\eta} \E_{x\sim\mathcal D}\left [ \Phi'\left(  \E_{\gamma}[\sigma(\theta_\eta,x)]  \right) 
 \E_{\gamma}[\nabla \sigma(\theta_\eta,x) ^\top (\theta' - \tilde\theta) ] 
 \right ] \bigg |_{\eta=0}  \notag  \\
& =  \E_{x\sim\mathcal D} \left [  \Phi''\left(  \E_{\gamma}[\sigma(\theta_\eta,x)]  \right) 
 (\E_{\gamma}[\nabla \sigma(\theta_\eta,x)  (\theta' - \tilde\theta) ])^2
 \right ] \notag  \\
 & ~~~~~~~~~~
 +~~\E_{x\sim\mathcal D}\left [ \Phi'\left( \E_{\gamma}[\sigma(\theta_\eta,x)]  \right) 
 \E_{\gamma}[(\theta'-\tilde\theta)^\top\nabla^2 \sigma(\theta_\eta,x)  (\theta' - \tilde\theta) ] 
 \right ]  
 \bigg |_{\eta=0} \label{equ:goshD2}  \\
 & = 0 
 ~~ +~~
 \E_{\gamma}[(\theta'-\tilde\theta)^\top S_\rho(\theta) (\theta' - \tilde\theta) ] \notag  \\
 & = \E_{\theta\sim \rho}[\trace(S_\rho(\theta)
 \Sigma_{\rho,\rho'} (\theta)) ]. \notag 
\end{align}
Further, 
we can show that $\nabla_{\eta\eta\eta}^3 \L[\rho_\eta] \bigg |_{\eta = \xi} = \obig(\D_\infty(\rho,\rho')^3)$, since when taking the third gradient, all the terms of the derivative are bounded by $\norm{\theta-\theta'}^3$. Specifically, taking the derivative of the form of  $\nabla_{\eta\eta}^2\Lm[\rho_\eta]$ in \eqref{equ:goshD2} gives 
$$
\begin{aligned} 
& \nabla_{\eta\eta\eta}^3 \L[\rho_\eta] \bigg |_{\eta = \xi}   \\ 
& =  \E_{x\sim\mathcal D} \left [ \Phi'''\left(  \E_{\gamma}[\sigma(\theta_\xi,x)]  \right) 
 (\E_{\gamma}[\nabla \sigma(\theta_\xi,x)  (\theta' - \tilde\theta) ])^3
 \right ]  \\
 & ~~~~~~~~~~  
  +~~3\E_{x\sim\mathcal D} \left [ \Phi''\left(  \E_{\gamma}[\sigma(\theta_\xi,x)]  \right) 
 \E_{\gamma}[\nabla \sigma(\theta_\xi,x)  (\theta' - \tilde\theta) ]
 \E_{\gamma}[(\theta'-\tilde \theta)^\top\nabla^2 \sigma(\theta_\xi,x)  (\theta' - \tilde\theta) ] 
 \right ]  \\
 & ~~~~~~~~~~
 +~~\E_{x\sim\mathcal D}\left [\Phi'\left( \E_{\gamma}[\sigma(\theta_\xi,x)]  \right)
 \E_{\gamma}[ \langle \nabla^3 \sigma(\theta_\xi,x), ~ (\theta'-\tilde \theta)^{\otimes 3}\rangle  
 ] 
 \right ]   \\  
 & = \obig\left (\esssup_{(\theta,\theta')\sim \rho} \norm{\theta'-\tilde \theta}^3\right) \\
 & = \obig(\D_\infty(\rho,\rho')^3). 
\end{aligned}
$$
Here we use the notation $\langle A, ~ v^{\otimes 3} \rangle =  \sum_{ijk=1}^d A_{ijk} v_i v_j v_k$. 
This completes the proof. 
\end{proof}

\begin{proof}[\textbf{Proof of Theorem~\ref{thm:wassdescent}}]
Following Theorem~\ref{thm:wasseroneorder}, we have
$$
\Delta^*(\rho,\epsilon) = \min_{\rho'} 
\left \{
\E_{\theta\sim \rho }\left [
G_\rho(\theta) ^\top \mu_{\rho,\rho'}(\theta) \right ] \colon ~ ~~
\D_{\infty} (\rho,\rho')
\leq  \epsilon \right \} + \obig(\epsilon^2). 
$$
For $\D_\infty (\rho,\rho')\leq \epsilon$, we must have $\norm{\mu_{\rho,\rho'}}\leq \epsilon$, and hence $\E_{(\theta,\theta')\sim \gamma_{\rho,\rho'}}\left [
G_\rho(\theta) ^\top \mu_{\rho,\rho'}(\theta) \right ] \geq - \epsilon\E_{\rho}[\norm{G_\rho(\theta)}]$ by Cauchy–Schwarz inequality. 
On the other hand, this minimum is achieved when $\mu_{\rho,\rho'} = -\epsilon G_\rho(\theta)/\norm{G_\rho(\theta)}$.  The only distribution $\rho'$ that satisfies this condition is $\rho' = (I - \epsilon G_\rho(\theta)/\norm{G_\rho(\theta)})\sharp \rho$. This proves Theorem~\ref{thm:wassdescent}a. 

For Theorem~\ref{thm:wassdescent}b, we need to use the result in Theorem~\ref{thm:wasser}, which yields, in the case of stable local optima, that 
$$
\Delta^*(\rho,\epsilon) = 
\min_{\rho'}\left \{\E_{\theta\sim \rho} \left [ \frac{1}{2} \trace\big ( S_\rho(\theta)^\top 
\Sigma_{\rho,\rho'}(\theta) \big )\right ]  
\colon ~ ~~
\D_{\infty} (\rho,\rho')
\leq  \epsilon 
\right\}~+~ \obig(\epsilon^3). 
$$
Similar to the argument above, the minima  should satisfy  
 $\Sigma_{\rho,\rho'}(\theta) \propto v_{min} v_{min}^\top$, where $v_{min}$ is the eigenvector of $S_{\rho}(\theta)$ associated with its minimum eigenvalue. 
 This corresponds to splitting $\theta$ into two copies with each weights with parameter $\theta \pm \epsilon v_{min}$ when $\lambda_{min} <0$, or keep $\rho$ unchanged when $\lambda_{min}>0$.  
\end{proof}

\clearpage\newpage

\section{Experimental Settings and Additional Results}

\subsection{Two-Layer RBF Neural network}  
\label{sec:app_two_layer_rbf_nn}

We consider fitting a simple radial basis function (RBF) neural network of form 
$$
f(x) = 
\sum_{i=1}^{m}  \sigma(\theta_i, x), 
~~~~~~~~\sigma(\theta, x) := \theta_{i,3} \times \exp\left (-\frac{1}{2}(\theta_{i,1}x + \theta_{i,2})^2 \right),
$$
where $x\in \RR$ and $\theta_i = [\theta_{i,1}, \theta_{i,2},\theta_{i,3}]^\top \in \RR^{3}$.  
For the ground truth, we set $m = 15$ and sample the true values of parameters 
$\{\theta_i\}$ 
from $\normal(0, 3)$, yielding the light blue curves shown in Figure~\ref{fig:app_nn_toy}.  
We generate a training data set $\mathcal D:= \{x_i, y_i\}_{i=1}^{1000}$ by drawing $x_i$ from $\mathrm{Uniform}[-5,5]$  and set $y_i =f_*(x_i)$ without noise, where $f_*$ denotes the true network we sampled. 
%
The network is trained by minimizing the mean square loss: 
$$
\min_{f} \E_{x
\sim \mathcal D} \left [(f_*(x) - f(x))^2 \right ]. 
$$
Mapping to \eqref{equ:LLtheta}, we have $\Phi(f) = (f_*-f)^2$. 
We learn the function using our splitting method and other progressive training baselines, all starting from $m=1$ neuron. 
We add one additional neuron in each splitting/growing phase for all the methods.
The parametric descent phase is performed using typical stochastic
gradient descent until convergence.
We stop the splitting process at $m = 8$ for all the methods. 
Figure \ref{fig:app_nn_toy} shows 
curves learned by different methods 
with $m=3$ and $m=8$ neurons, respectively.
Our method yields better approximation.
\todo{but the case when $m=1$ should be the same across all the methods?}
\todo{how about having two plots, showing the curves of different method when $m=3$ and $m=8$, respectively?}
\todo{better line styles and legend fix ``$f(x)$'' to "True"} 
\todo{Without any other implementation details, it gives impression of careless. Given some more implementation details? 
Stepsize, iteration, stopping criteria, etc?
}


\begin{figure*}[ht]
\centering
\setlength{\tabcolsep}{1pt}
\scalebox{1.18}{
\begin{tabular}{cc}
\raisebox{4.3em}{\rotatebox{90}{\small $y$}}
\includegraphics[height =0.22\textwidth]{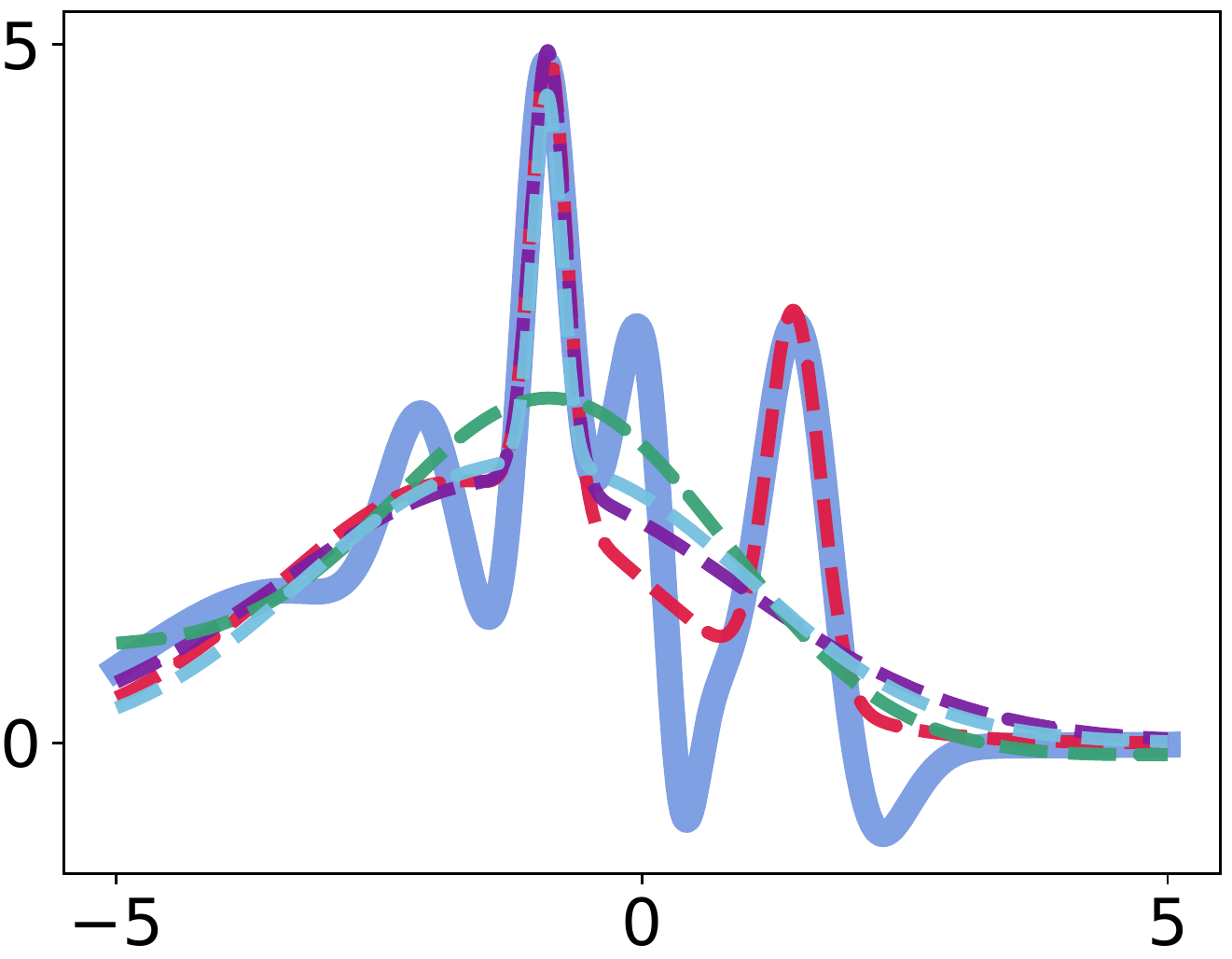} &  ~
\includegraphics[height =0.22\textwidth]{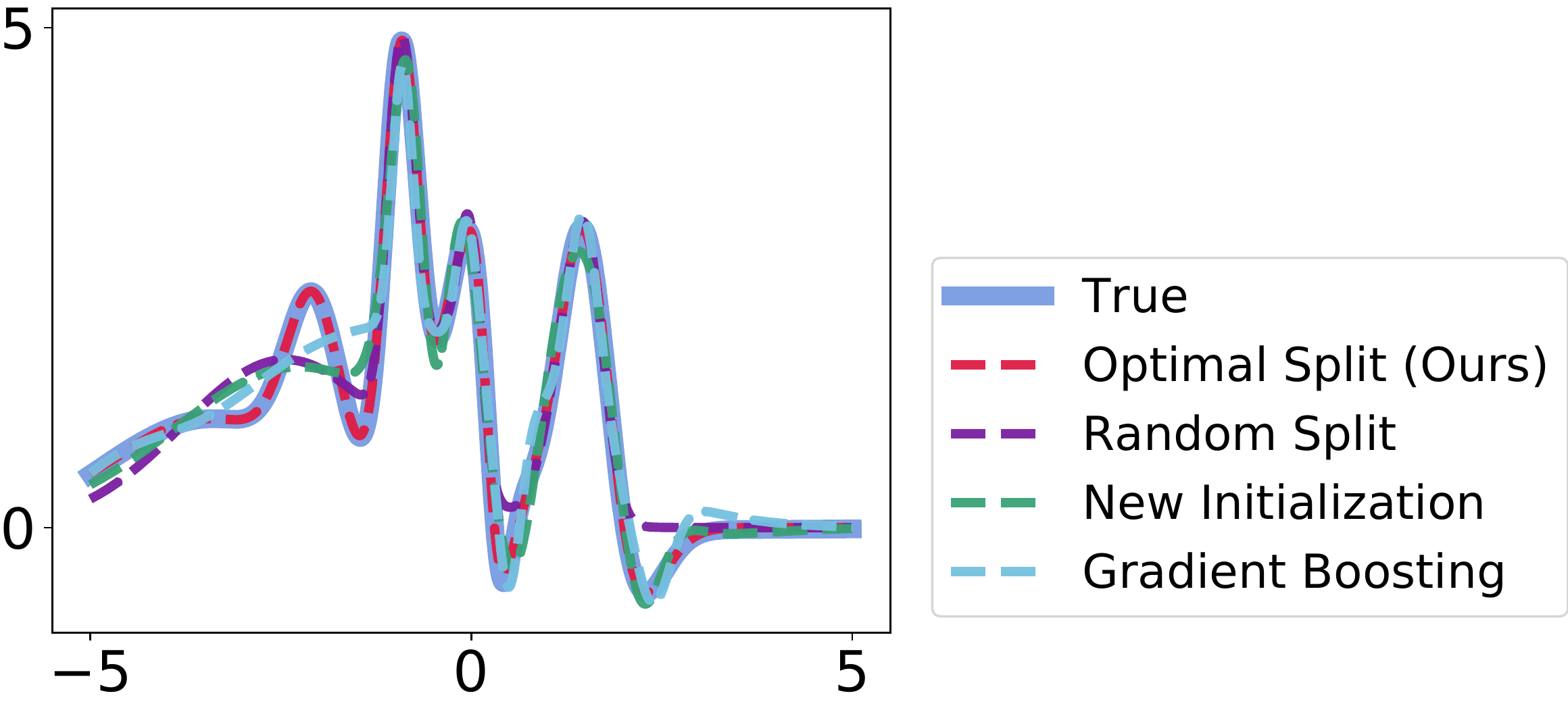}  \\ 
{\scriptsize $x$} & \hspace{-8em}{\scriptsize $x$}  \\ 
{\scriptsize (a) $m = 3$} & {\scriptsize \hspace{-12em} (b) $m = 8$}    \\
\end{tabular}
}
\caption{\small Results on the toy RBF neural network. }
\label{fig:app_nn_toy}
\end{figure*}

\clearpage \newpage
\subsection{Learning Interpretable Neural Network} 
\label{sec:app_explainable_nn}
We provide more details on learning the interpretable neural network.

\paragraph{Setting}
We adopt the interpretable neural architecture proposed in \citet{li2018deep} as our testbed.
Unlike standard black-box neural networks, 
this architecture contains a 
 special prototype layer in the classifier, 
 which includes a set of prototype neurons that are enforced to encode to realistic images for promoting interpretability.  
 In this model, each input image $x$ is first mapped to a lower-dimensional representation based on its distance $\norm{ \theta - e(x)}$ with a set of prototype vectors, 
 where $\theta\in \R^{40}$ represents a prototype vector and  $e(x)$ is an encoder function.  
  The prototype vectors are enforced to be interpretable in that they can be decoded to some realistic images; this is achieved in \citet{li2018deep} by introducing a regularization term that minimizes the minimum square distance between the prototypes and the training data, that is, 
 {$\min_{i}\norm{\theta -e(x_i)}$}, where $\{x_i\}$ denotes the training dataset. 
 
 We apply our  method to split the prototype neurons, 
 by treating $\sigma(\theta,x) := \norm{\theta- e(x)}$ as the activation function. 
 We use the MNIST dataset in our experiment. 
We visualize the prototype neurons we learned using the images that they encode,  by feeding the prototype vectors $\theta$ into a decoder function jointly trained  with the network. 
We use the same encoder and decoder architectures, as suggested in \citet{li2018deep} and 
refer the reader to \citet{li2018deep} for more implementation details. 
To better understand the splitting dynamics, we start with a  small network with just one prototype neuron and  
gradually add more prototypes via splitting. 

We  compare our method with  two baseline methods, 
 {\tt New Initialization} and {\tt Random Split}, that also progressively grow the prototype layers starting from one prototype neuron.   
In {\tt New Initialization}, we simply add one new prototype neuron with random initialization at each iteration. 
In {\tt Random Split}, we randomly pick a prototype neuron to split and split it following its splitting gradient given by our splitting matrix.
Figure~\ref{fig:app_mnistvisual} visualizes the full splitting/growing process of our method and the two baselines. 
We can see that our splitting method successfully identifies the most ambiguous (and least interpretable) prototype neurons to split at each iteration, and achieves the best final results.

\begin{figure*}[t]
\centering
\begin{tabular}{c}
Optimal Split (Ours) \\
\includegraphics[width =0.7\textwidth]{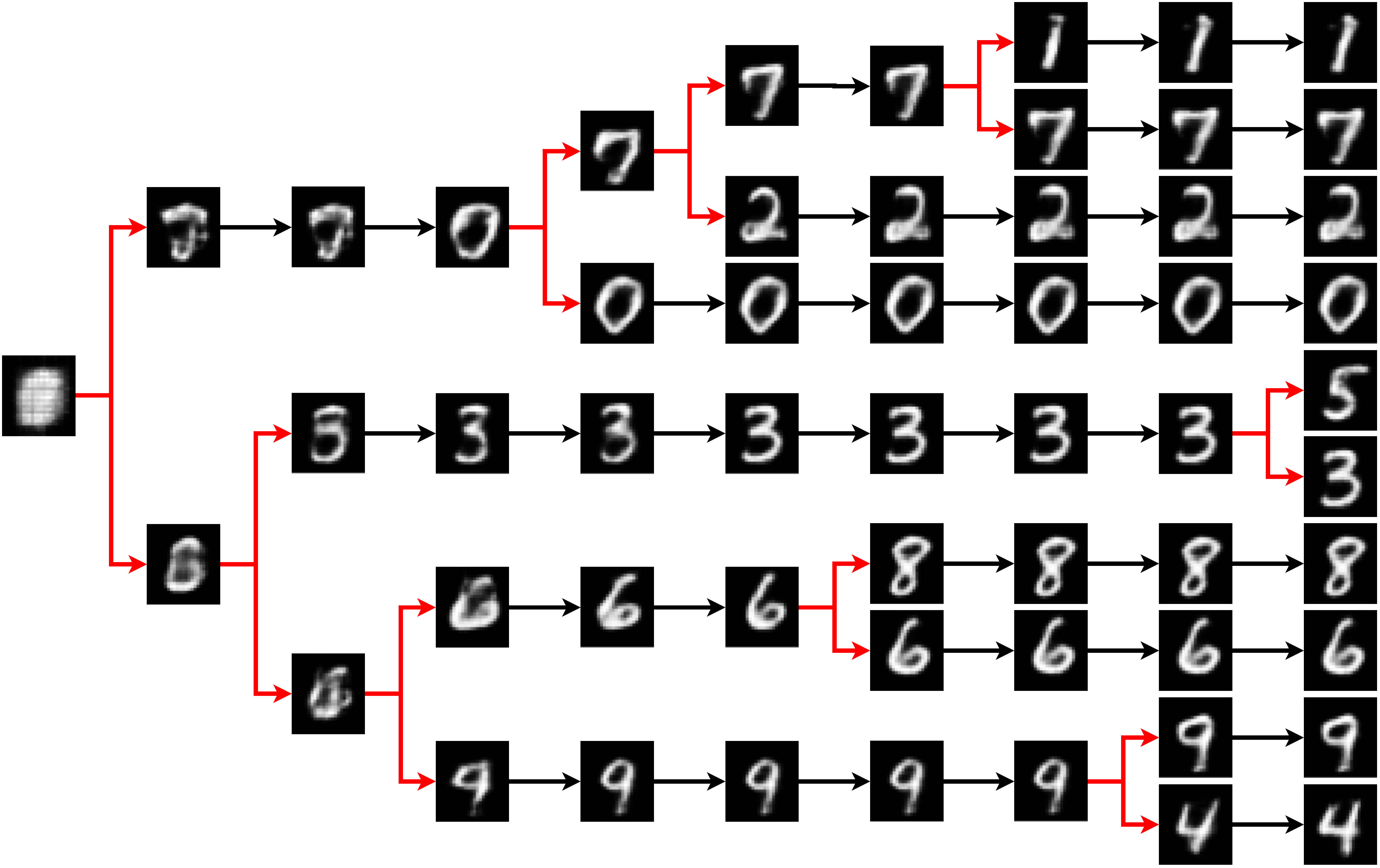} \\\\
Random Split \\
\includegraphics[width =0.7\textwidth]{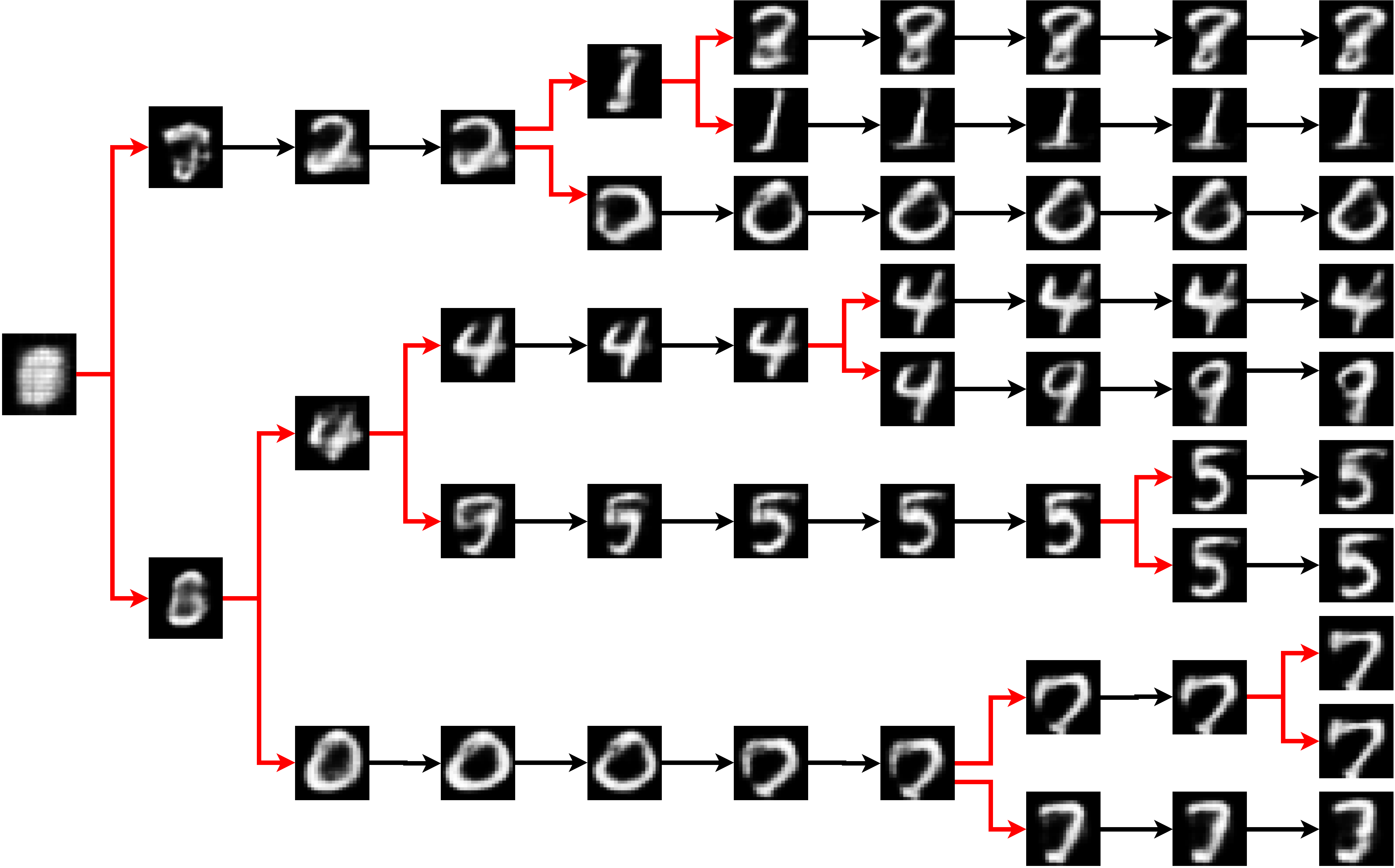} \\\\
New Initialization \\
\includegraphics[width =0.7\textwidth]{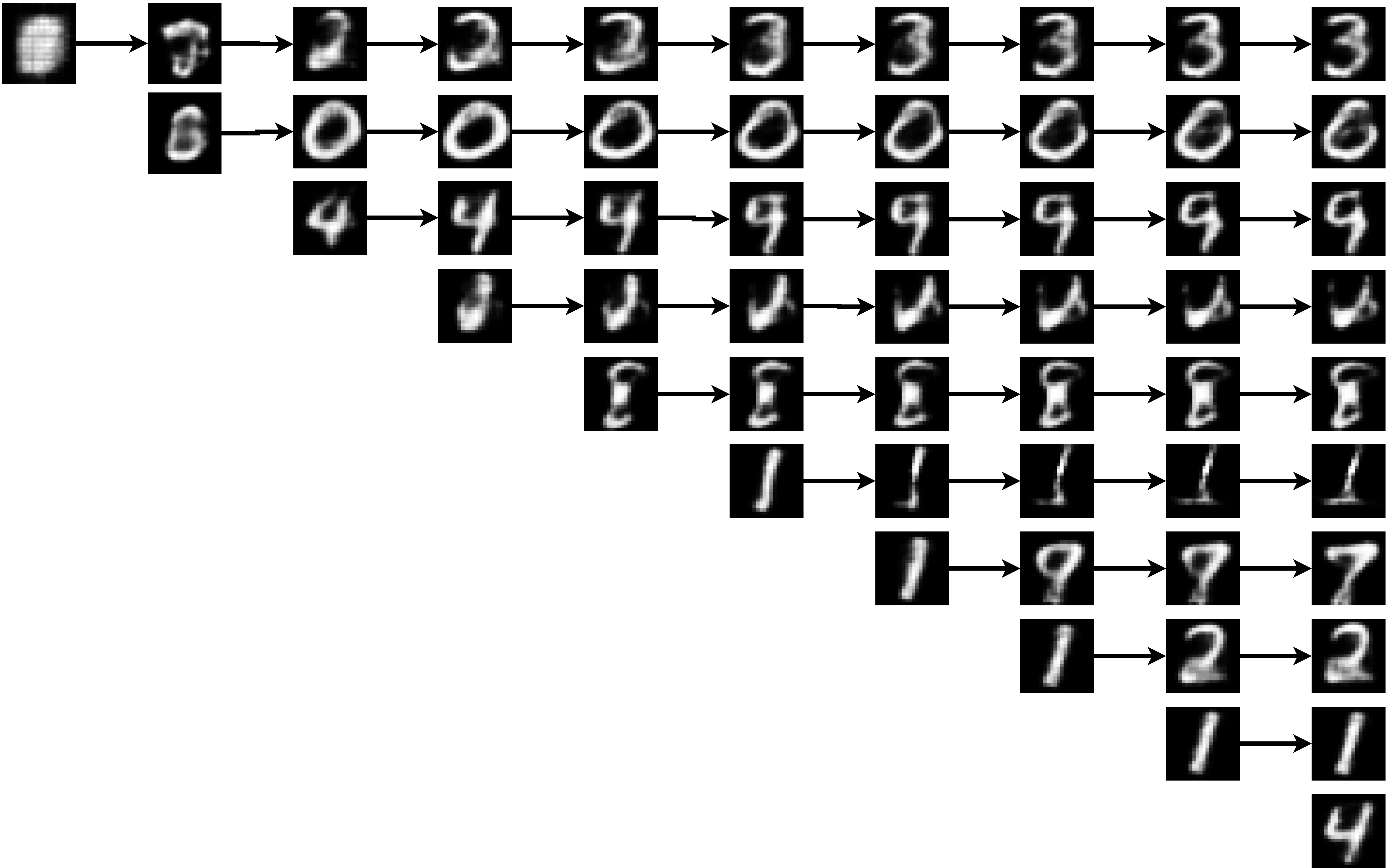} \\\\
\end{tabular}
\caption{Visualizing the growing process of the prototype neurons given by our splitting method and the two baselines.}
\label{fig:app_mnistvisual}
\end{figure*}

\clearpage \newpage
\subsection{Lightweight Neural Architectures for Image Classification}
\label{sec:image_classfication}
We 
describe details of our experiments 
on learning lightweight deep networks for image classification.

\paragraph{Dataset and Backbone Networks}
We use the CIFAR-10 benchmark dataset.
We adopt a standard data argumentation scheme (mirroring and shifting) that is widely used for this dataset \citep{liu2018rethinking, liu2017learning}.
The input images are normalized using channel means and standard derivations.
We use two popular deep neural architectures as our testbed,  MobileNet~\citep{howard2017mobilenets} and  VGG19~\citep{simonyan2014very}.

\paragraph{Training Settings}
We treat the {filters} as the neurons to split for convolutional neural networks. 
For example, consider a convolutional layer with $n_{out} \times  n_{in} \times k \times k$ parameters, where $n_{out}$ denotes the number of output channels and $n_{in}$ the number of input channels and $k$ the filter size. 
We treat it as $n_{out}$ neurons, and each neuron has a parameter of size $n_{in} \times k \times k$. 
To apply our methods, 
we start with a small variant of the MobileNet and VGG19, 
and gradually grow the network by splitting the (convolutional) neurons with the most negative splitting indexes following Algorithm \ref{alg:main}.
For MobileNet, we construct the initial network by keeping the size of the 
first convolution layer as the same (=32) as the original MobileNet
and setting the number of depthwise and pointwise channels to be 16.
For VGG19, we set the number of {channels} of the initial network to be 16 for all layers.

For the parametric descent phase, we use stochastic gradient descent
with an initial learning rate $0.1$ for 160 epochs.  
The learning rate is divided by 10 at 50\% and 75\%
of the total number of training epochs. 
We use a weight decay of $10^{-4}$ and a Nesterove momentum of 0.9 without dampening.
The batch size is set to be 64. 
In each splitting phase, 
we increase the number of channels by a percentage of 30 using our method.

Note that our splitting matrix (see Eq.~\ref{equ:sell}) involves the second-order
derivative of the activation function, which is not well defined for ReLU activation. 
Therefore, we replace the ReLU activation with Softplus to prevent numerical issues in calculating the splitting matrices. We also apply Softplus in the other experiments that contain ReLU activation function in the network.

\paragraph{Pruning}
We compare with two model pruning algorithms:  
the batch-normalization-based pruning (Bn-prune) by \citet{liu2017learning} 
and the L1-based pruning (L1-prune) by \citet{li2016pruning}.  
Bn-prune imposes L1-sparsity on the channel-wise scaling factors in the batch normalization layers during training,  and prunes channels with lower scaling factors afterwards. 
L1-prune removes the filters with weights of small L1-norm in each layer.  
%
For both pruning baselines, we use the implementation provided by \citet{liu2018rethinking}. 
For Bn-prune, we set the sparsity term to be 0.0001 for all the cases.
We initial both pruning methods from a full-size backbone network (MobileNet and VGG19) that we trained starting from scratch. 
After each pruning phase, the parameters of the pruned network are finetuned starting from the previous values using stochastic gradient descent, following the same setting as that we use in splitting steepest descent.

\paragraph{Finetuning vs. Retraining}
In both the splitting and pruning methods above, 
the parameters of the split/pruned networks are successively  \emph{finetuned} starting from  the previous values.  
In order to test the performance of the network architectures 
given by both splitting and pruning methods, 
we test another setting in which we \emph{retrain}
the network parameters after each splitting/pruning step, that is, we discard all the parameters of the network, and retrain the whole network starting from a random initialization, under the network structure obtained from splitting or pruning at each iteration.   
As shown in Figure~\ref{fig:cifar10}c-d, the results of retraining 
is comparable with (or better than) the result of successive finetuning in Figure~\ref{fig:cifar10}a-b, which is consistent with the findings in \citet{liu2018rethinking}.

\subsection{Resource-Efficient Keyword spotting}
\label{sec:app_kws}
We apply our methods on the application of keyword spotting.
Keyword spotting systems aim to detect a particular set of keywords 
from a continuous stream of audio, which is typically deployed on a wide range of edge devices with resource constraints. 

\paragraph{Dataset and Training Settings} 
We use the Google speech commands benchmark dataset \citep{warden2018speech} for comparisons.
We are interested in the setting that the model size is limited to less than 500K
and adopt the optimized architectures with tight resource constraints provided in \citet{zhang2017hello}
as our baselines.
For fair comparison, 
we closely follow the experimental settings described in \citet{zhang2017hello}.
We split the dataset into 80/10/10\% for training, validation and test, respectively.


We start with a very narrow network and progressively grow it using splitting steepest descent. 
We build our initial narrow network based on the DS-CNN architecture proposed in \citet{zhang2017hello}, by reducing the number of channels in each layer to 16. 
The backbone DS-CNN model consists of one regular convolution layer and five depthwise and pointwise convolution layers \citep{howard2017mobilenets}. We refer the reader to \citet{zhang2017hello}
for more information. 
At each splitting stage, 
we increase the number of channels by a percentage of 30\% using the approach
described in Algorithm~\ref{alg:main}.
We use the same hyper-parameters for training and evaluation as in \citet{zhang2017hello}.

\subsection{Splitting Steepest Descent for Minimizing MMD}  
\label{sec:toy_mmd}
We consider the problem of data compression. 
Given a large set of data points $\{\theta_i^*\}_{i=1}^N$, we want to find a smaller set of points $\{\theta_i\}_{i=1}^n$, 
equipped with a set of importance weights $\{w_i\}_{i=1}^n$, 
to approximate the larger dataset. This problem can be solved by minimizing maximum mean discrepancy (MMD) \citep{gretton2012kernel} using 
 conditional gradient method (a.k.a. Frank-Wolfe), 
an algorithm known as \emph{herding} 
\citep{chen2012super, bach2012equivalence}.  
In this section, we provide additional results on using splitting steepest descent to 
minimize MMD by progressively introducing new points via splitting.  

Denote by $\rho_* = \sum_{i=1}^N \delta_{\theta_i^*}/N$ the empirical distribution of the original dataset,  
and $\rho = \sum_{i=1}^n w_i \delta_{\theta_i}$ the (weighted) empirical distribution of the compressed data.
 Let $k(\theta,\theta')$ be a positive definite kernel, which can be represented using a random feature expansion of form 
 $$
 k(\theta,\theta') = \E_{x\sim \pi}[\sigma(\theta, x) \sigma(\theta', x)], 
 $$ 
  where $\sigma(\theta, x)$ is a feature map index by an auxiliary variable $x$, 
 and $\pi$ is a distribution on $x$.  
 The $\sigma(\theta, x)$ can be taken to be the cosine function for commonly used kernels such as RBF kernel; see \citet{rahimi2007random} 
 for more information on random feature expansion.   
Then the MMD between $
\rho$ and $
\rho^*$, with kernel $k(\theta,\theta')$, can be written into 
\begin{align} 
\MMD(\rho,\rho_*) 
& = \E_{\rho,\rho_*}[k(\theta, \theta') - 2k(\theta, \theta_*') + k(\theta_*, \theta_*') ] \notag \\
& = \E_{x\sim \pi}[(\E_{\theta\sim \rho}[\sigma(\theta, x)] - \E_{\theta_* \sim \rho_*}[\sigma(\theta_*, x)])^2],  \label{equ:herdingsquare}
\end{align}
where $\theta,\theta'$ are i.i.d. drawn from $\rho$ and $\theta_*, \theta_*'$ are i.i.d. drawn from $\rho_*$. The data compression problem can be viewed as minimizing the MMD: 
$$
\min_{\rho} \left \{ \L[\rho] := \MMD(\rho,\rho_*)  \right \}. 
$$
From \eqref{equ:herdingsquare}, 
this minimization can be viewed as 
performing least square regression 
on a one-hidden-layer neural network $f_\rho(x) = \E_{\theta\sim \rho}[\sigma(\theta,x)],$ where each data point $\theta_i$ is viewed as a neuron.  
Therefore, splitting steepest descent can be applied to minimize the loss function. 
This allows us to start with a small number of  data points (neurons), and gradually increase the number of points by splitting. 
The splitting matrix of $\L[\rho]$ is 
\begin{align*}
    S_\rho(\theta)
    & = 2\E_{x\sim \pi}\left [\left (\E_{\theta'\sim  \rho}[\sigma(\theta', x)] -  \E_{\theta_* \sim \rho_*}[\sigma(\theta_*, x)] \right )\nabla_{\theta\theta}^2 \sigma(\theta, x)\right ] \\
    & = 2\E_{\theta'\sim \rho,\theta_*\sim \rho_*}\left[\nabla_{\theta\theta}^2 k(\theta, \theta') - \nabla_{\theta\theta}^2 k(\theta, \theta_*) \right ].
\end{align*}

We apply splitting steepest descent ({\tt Optimal Split}) in 
Algorithm~\ref{alg:main} 
{starting from a single point (neuron)}.
We compare our method with {\tt Random Split}, {\tt Gradient Boosting} (a.k.a. Frank-Wolfe or herding),   
{\tt New Initialization}. 
In {\tt Random Split}, we randomly pick a point to split, and split it following its splitting gradient direction.  
In  {\tt Gradient Boosting}, 
a new point is introduced greedily at each iteration by minimizing the MMD loss, with all the previous points fixed. 
In {\tt {\tt New Initialization}},
a new random point is introduced and co-optimized together with all the previous points at each iteration.  

In our experiment, 
we construct $\rho_*$ by drawing  
an i.i.d. sample of size  $N=1000$ from a one-dimensional Gaussian mixture model  
$0.2\normal( -2, 0.5) + 0.3 \normal(1., 0.5) + 0.5 \normal(3, 0.5)$ as ground truth. 
We initialize all the methods from a same point drawn from $\mathrm{Uniform}[-5, -3]$, and add a new point in each splitting/growing phase. 
The parametric descent phase is performed using the adagrad optimizer with a constant learning rate $0.01$ for all the methods.  

Figure~\ref{fig:app_toy_mmd_split} plots the training dynamics of all the methods. 
The size of each dot represents the particle weight. 
Note that in {\tt Optimal Split} and {\tt Random Split}, 
each off-spring shares half of the weights of their parent points, but 
in \texttt{New Initialization} and \texttt{Gradient Boosting}, all the points evenly divide the weights all the time.

\begin{figure*}[ht]
\centering
\setlength{\tabcolsep}{1pt}
\begin{tabular}{ccccc}
\raisebox{0.3em}{\rotatebox{90}{\scriptsize \tt Optimal Split}}
\includegraphics[width =0.18\textwidth]{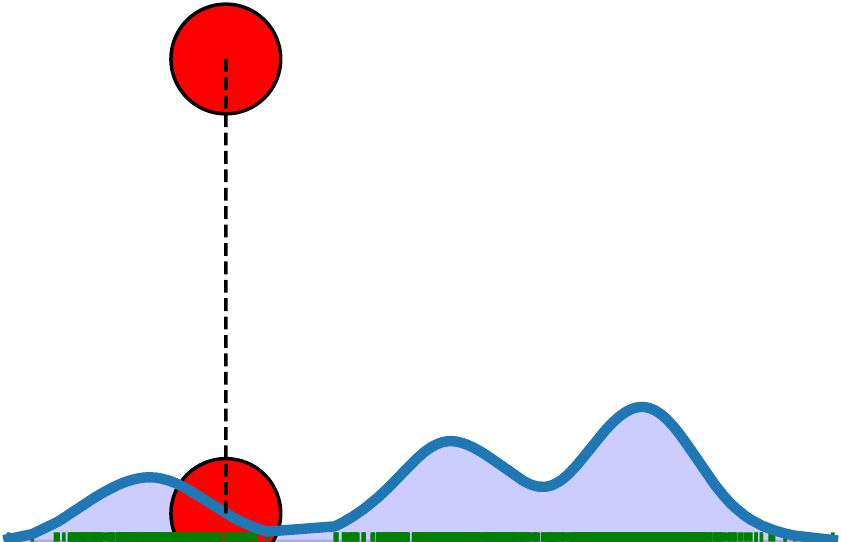} &
\includegraphics[width =0.18\textwidth]{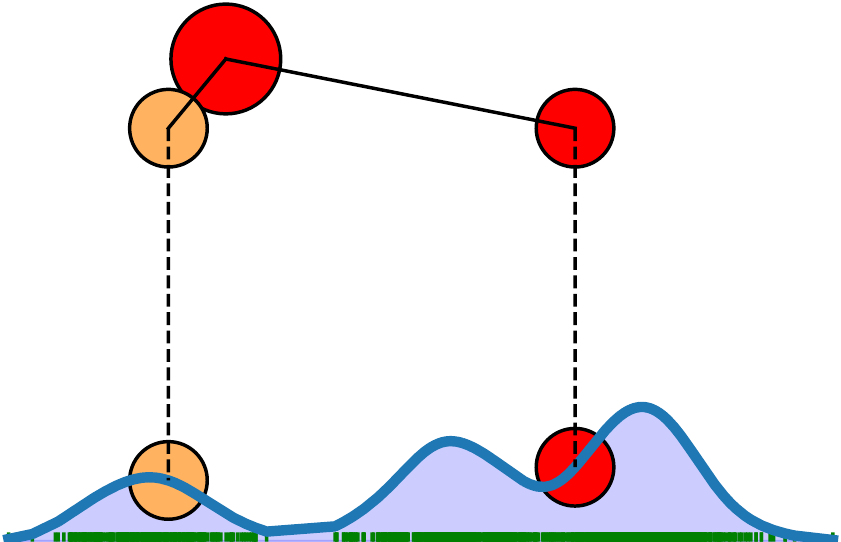} &
\includegraphics[width =0.18\textwidth]{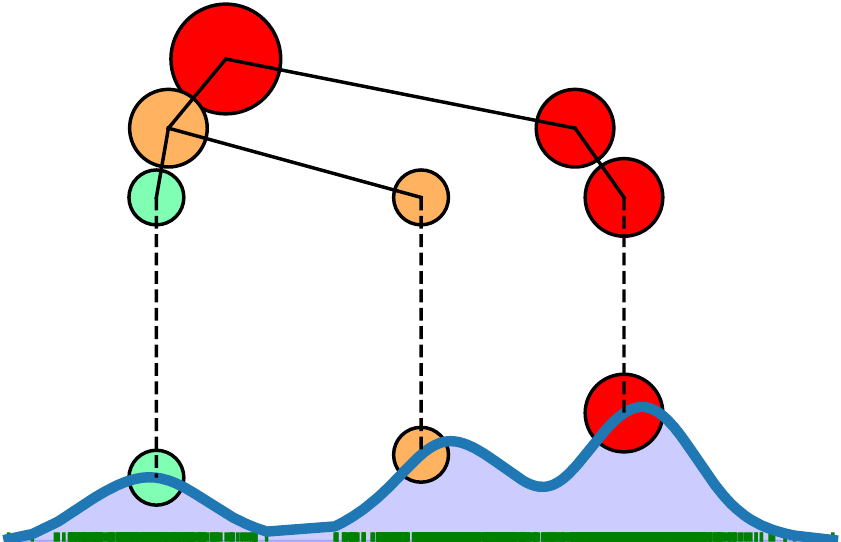} &
\includegraphics[width =0.18\textwidth]{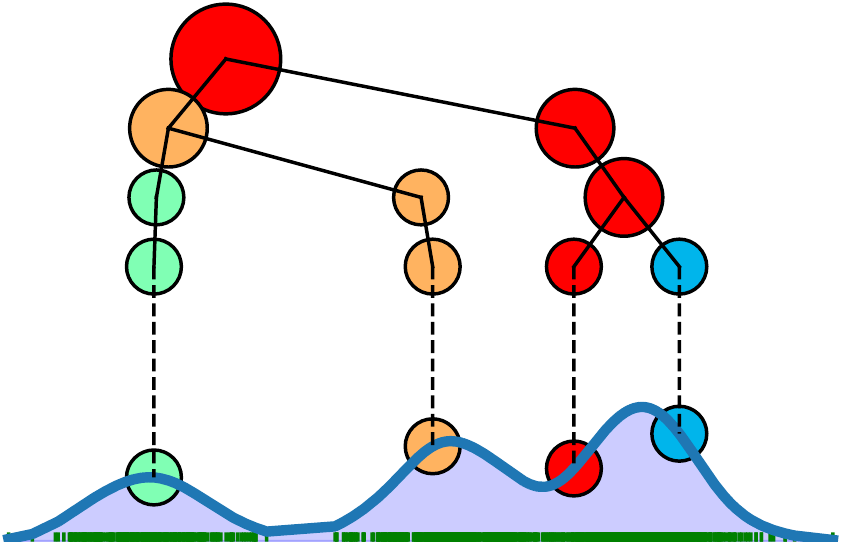} &
\includegraphics[width =0.18\textwidth]{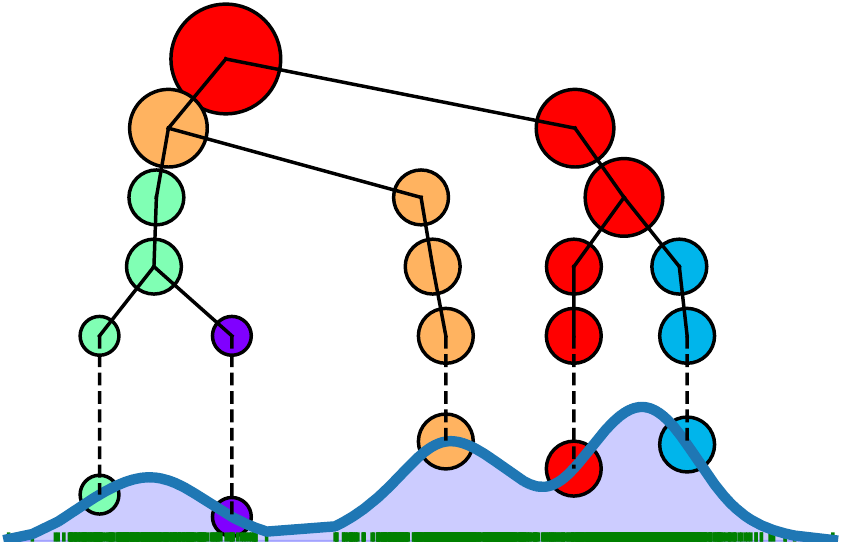} \\
\raisebox{0.3em}{\rotatebox{90}{\scriptsize \tt Random Split}}
\includegraphics[width =0.18\textwidth]{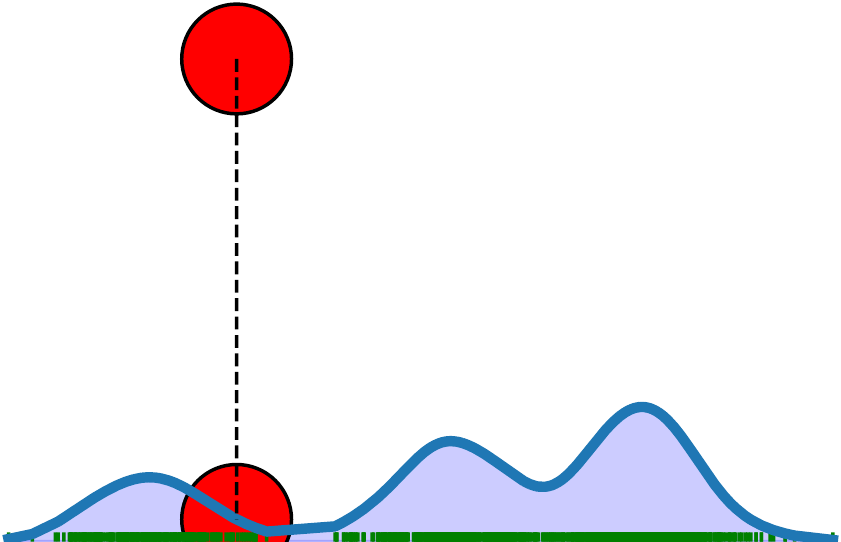} &
\includegraphics[width =0.18\textwidth]{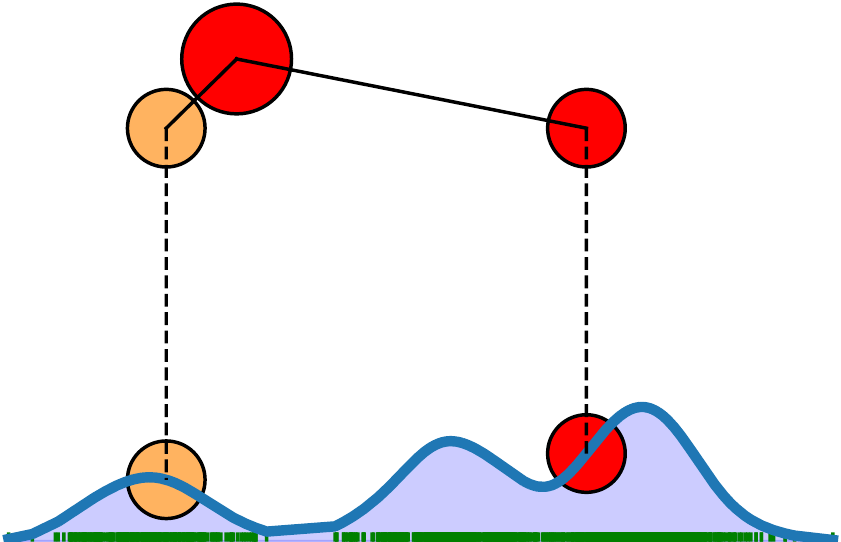} &
\includegraphics[width =0.18\textwidth]{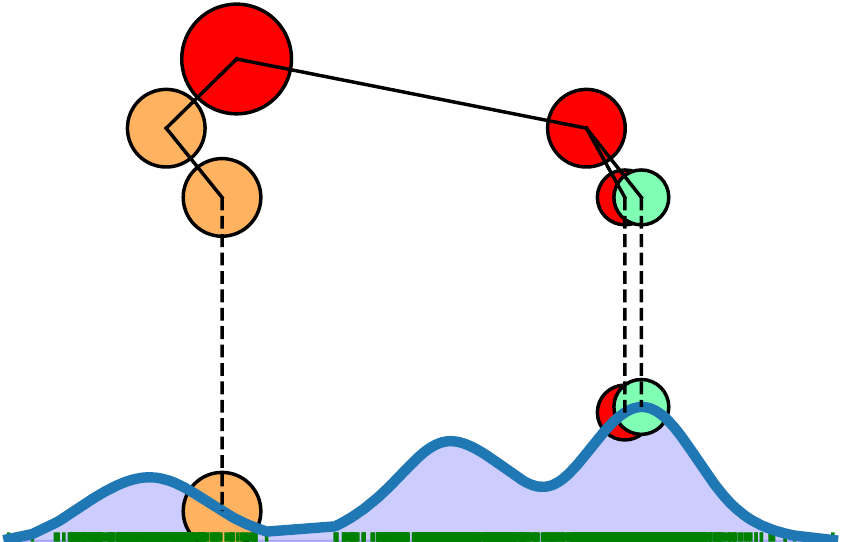} &
\includegraphics[width =0.18\textwidth]{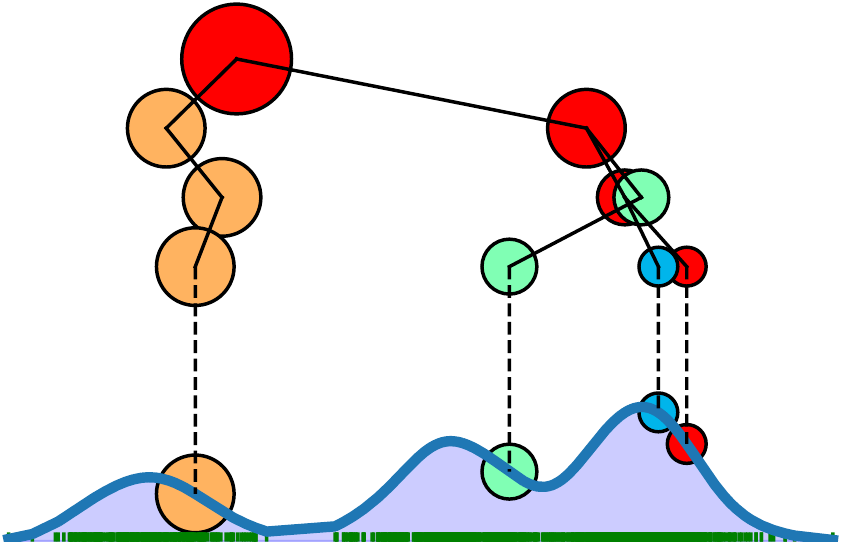} &
\includegraphics[width =0.18\textwidth]{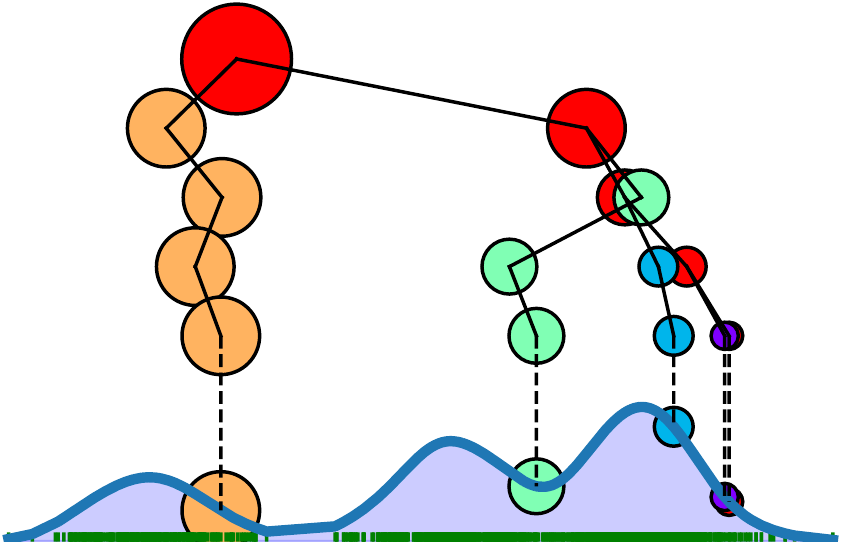} \\
\raisebox{0.3em}{\rotatebox{90}{\scriptsize \tt New Initialization}}
\includegraphics[width =0.18\textwidth]{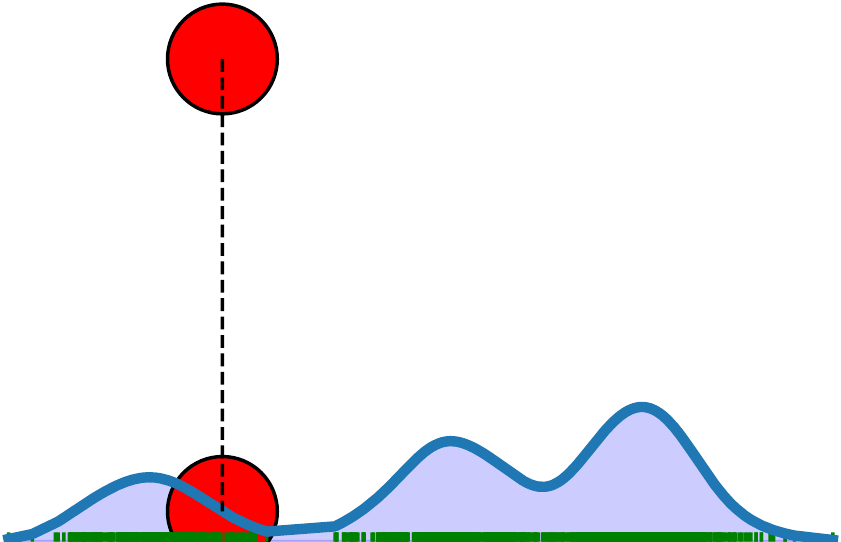} &
\includegraphics[width =0.18\textwidth]{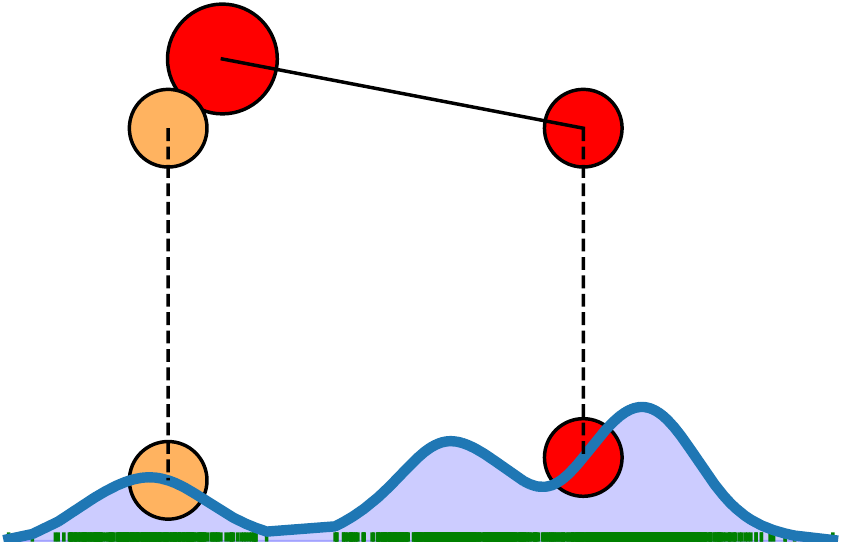} &
\includegraphics[width =0.18\textwidth]{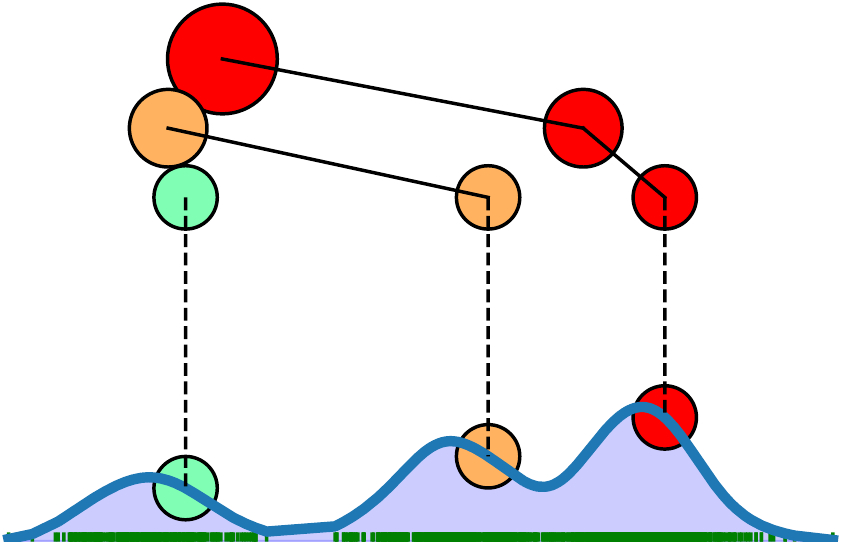} &
\includegraphics[width =0.18\textwidth]{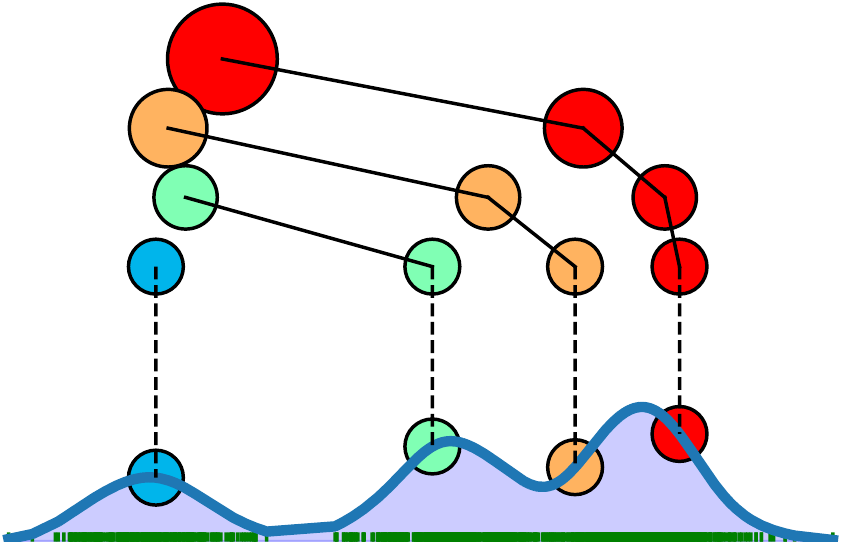} &
\includegraphics[width =0.18\textwidth]{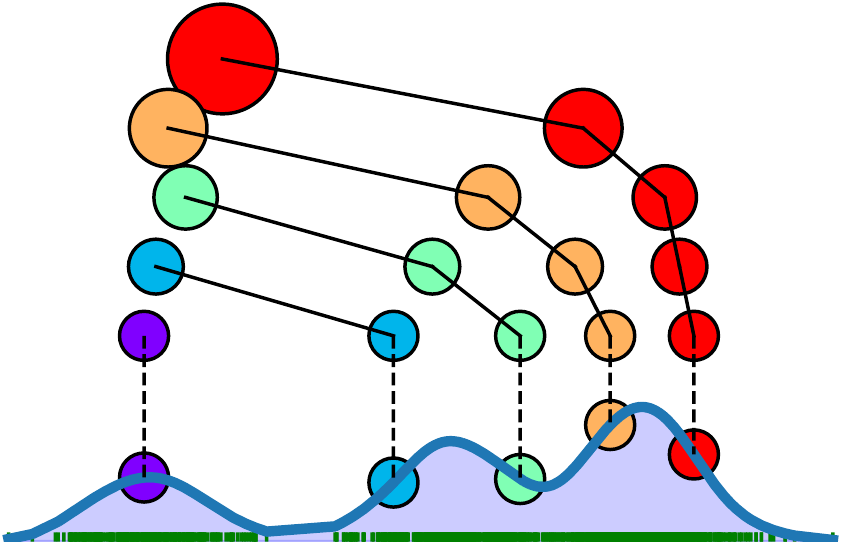} \\
\raisebox{0.1em}{\rotatebox{90}{\scriptsize \tt Gradient Boosting}}
\includegraphics[width =0.18\textwidth]{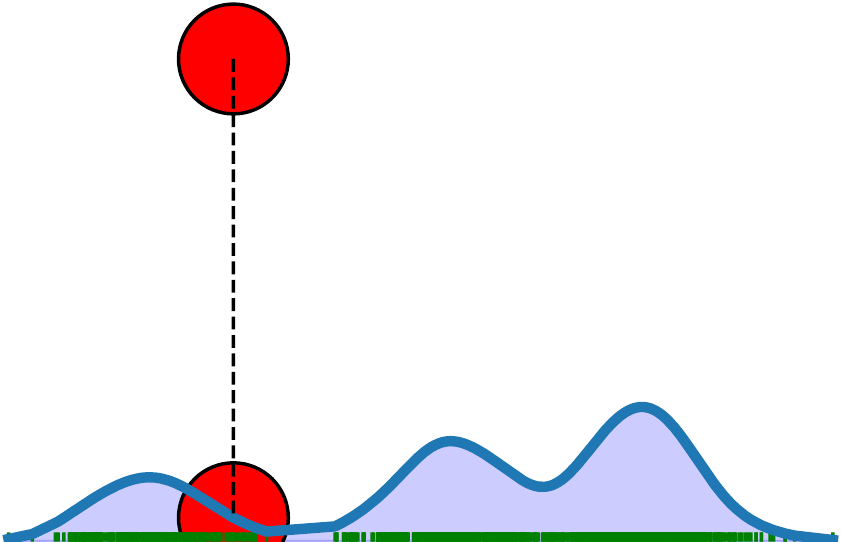} &
\includegraphics[width =0.18\textwidth]{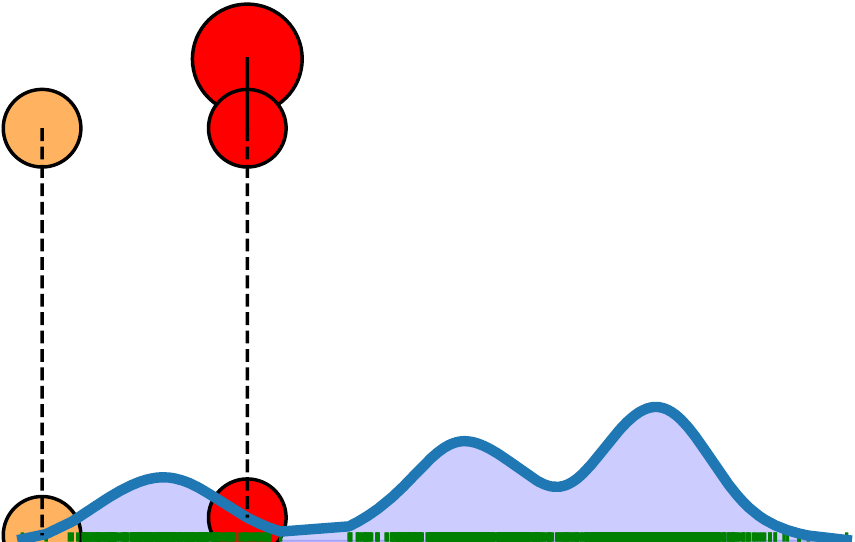} &
\includegraphics[width =0.18\textwidth]{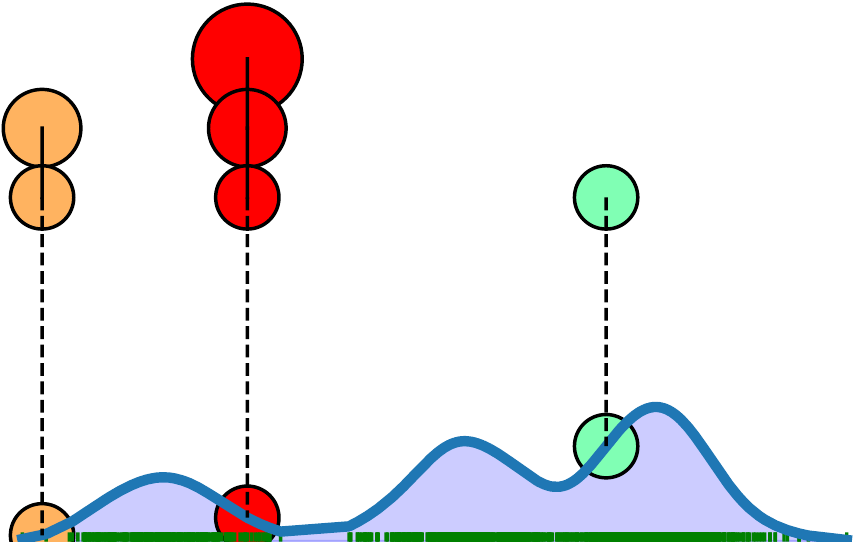} &
\includegraphics[width =0.18\textwidth]{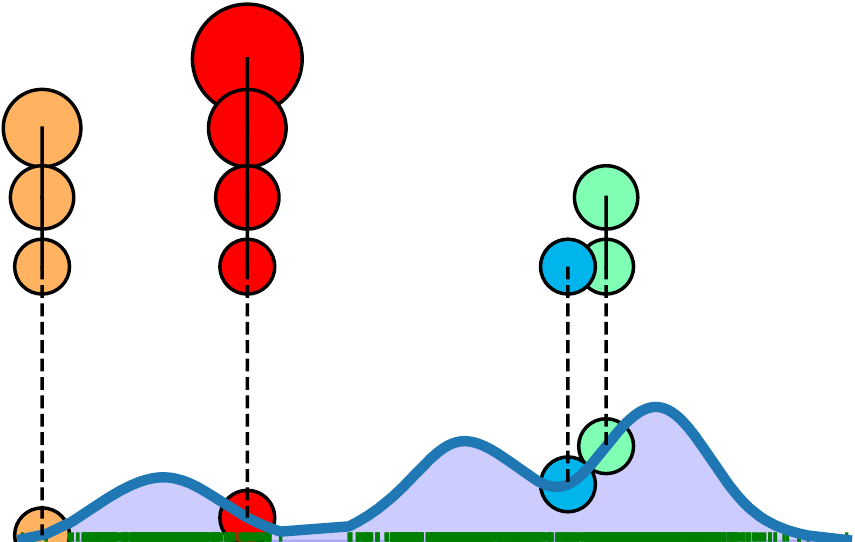} &
\includegraphics[width =0.18\textwidth]{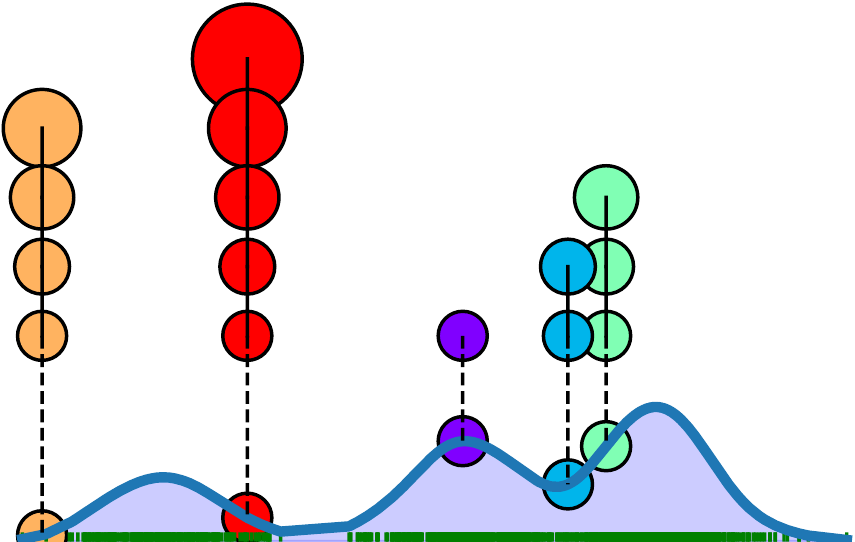} \\
{\small \#Particle=1} & {\small \#Particle=2} & {\small \#Particle=3 }& {\small \#Particle=4} & {\small \#Particle=5} \\
\end{tabular}
\caption{MMD minimization for data compression using different progressive optimization methods.} 
\label{fig:app_toy_mmd_split}
\end{figure*}

Figure~\ref{fig:app_loss_mmd}  shows the training iterations vs. the training loss (logarithm of MMD) of our method and the baseline approaches. 
As we can see from Figure~\ref{fig:app_loss_mmd}, our method yields the lowest training loss in general. 
The kicks of \texttt{New Initialization} and \texttt{Gradient Boosting} are resulted from 
re-weighting all particles after introducing new particles.

\begin{figure}
\centering
\begin{tabular}{c}
\raisebox{5em}{\rotatebox{90}{\small Log MMD}}~~
\includegraphics[height =0.4\textwidth]{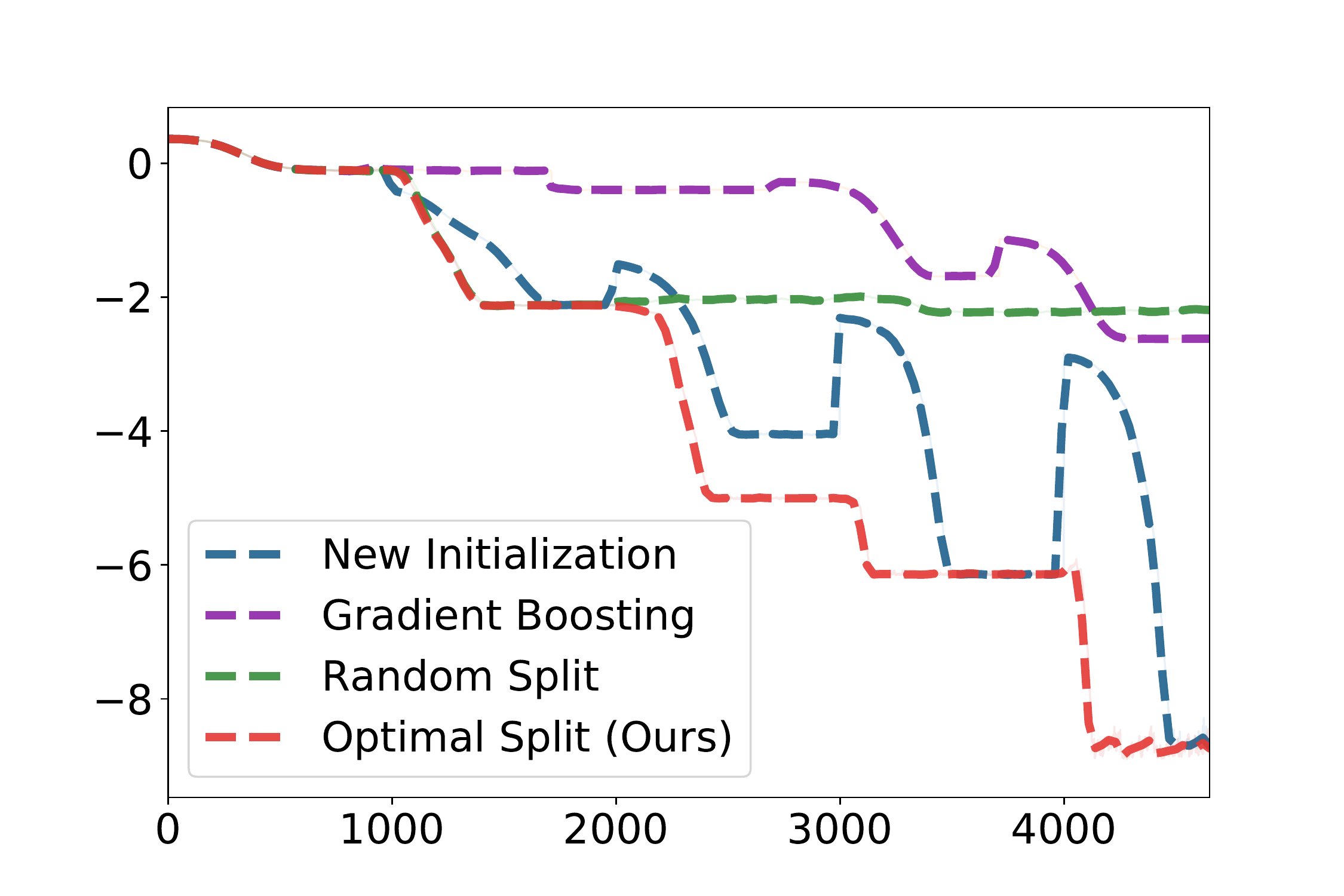} \\
~~~{\small Training Iterations} \\
\end{tabular}
\caption{Lose curve of different methods for MMD minimization.} 
\label{fig:app_loss_mmd}
\label{fig:app_transfer}
\end{figure}

\end{document}